\documentclass[journal,a4paper]{IEEEtran}

\usepackage{graphicx}
\usepackage{bm}
\usepackage{amsfonts}
\usepackage{amsmath}
\usepackage{amsthm}
\usepackage{amssymb}
\usepackage{bbm}
\usepackage{times}
\usepackage[export]{adjustbox}

\ifCLASSOPTIONcompsoc
\usepackage[subrefformat=parens,labelformat=parens,caption=false,font=normalsize,labelfont=sf,textfont=sf]{subfig}
\else
\usepackage[subrefformat=parens,labelformat=parens,caption=false,font=footnotesize]{subfig}
\fi

\usepackage{latexsym,bm,amsmath,amssymb} 
\usepackage{CJK}
\usepackage{hhline}
\usepackage[pagebackref=false]{hyperref}
\usepackage{siunitx}

\usepackage{multirow} 
\usepackage{xcolor}
\usepackage{epstopdf}
\usepackage{cite}
\usepackage[noend]{algpseudocode}
\usepackage{algorithmicx,algorithm}
\usepackage{comment}
\usepackage{bbm}

\usepackage{tabularx} %
\usepackage{enumitem} %

\newtheorem{theorem}{Theorem}
\newtheorem{lemma}{Lemma}
\DeclareMathOperator*{\argmax}{arg\,max}

\usepackage{mathtools}

\newcommand{\sinc}{\operatorname{sinc}}

\begin{document}
\title{Short-Length Code Designs for Integrated Sensing and Communications: A Deep Learning Approach}

\author{Muah Kim,~\IEEEmembership{Student Member,~IEEE,} Shuangyang Li,~\IEEEmembership{Member,~IEEE,}
Tayyebeh Jahani-Nezhad,~\IEEEmembership{Member,~IEEE,} 
Rafael F. Schaefer,~\IEEEmembership{Senior Member,~IEEE,} and
Giuseppe Caire,~\IEEEmembership{Fellow,~IEEE}

\thanks{
Part of the paper was presented at the \textit{IEEE International Conference on Communications}, Denver, Colorado, June. 2024 \cite{kim2024short}.\\
The work of T. Jahani-Nezhad was supported by the Gottfried Wilhelm Leibniz-Preis 2021 of the German Research Foundation (DFG).
The work of G. Caire and S. Li was supported by the German Federal Ministry of Education and Research Germany (BMBF) in the program of ``Souver{\"a}n. Digital. Vernetzt.'' Joint Project 6G-RIC (Project IDs 16KISK030). In addition, the work of S. Li was also supported in part by the European Union's Horizon 2020 Research and Innovation Program under MSCA Grant No. 101105732 – DDComRad. 
The work of R. F. Schaefer was supported in part by the BMBF through the Transfer Hub 6G-life under Grant 16KIS2413K and in part by the DFG as part of Germany's Excellence Strategy -- EXC 2050/2 -- Project ID 390696704 -- Cluster of Excellence \emph{``Centre for Tactile Internet with Human-in-the-Loop'' (CeTI)} of Technische Universit\"at Dresden. \\
M. Kim and R. F. Schaefer are with the Faculty of Electrical and Computer Engineering, Dresden University of Technology, Dresden 01069, Germany (e-mail:\{muah.kim, rafael.schaefer\}@tu-dresden.de).\\
S. Li, T. Jahani-Nezhad, and G. Caire are with the Faculty of Electrical Engineering and Computer Science, Technical University of Berlin, Berlin 10587, Germany (e-mail:\{shuangyang.li, t.jahani.nezhad, caire\}@tu-berlin.de).}}

\maketitle

\begin{abstract}
Integrated sensing and communication (ISAC) enables joint communication and sensing using a shared waveform, but its signal design is challenging due to the inherent trade-off between the two objectives, particularly in the short blocklength regime. This paper proposes an autoencoder (AE)-based framework for ISAC waveform design in noncoherent settings.

We derive a modified Cramér–Rao bound for multi-target delay estimation and analyze the maximum-likelihood decoding rule for noncoherent communication under correlated fading. These results reveal structural connections and trade-offs between communication and sensing objectives in waveform design.

Based on this analysis, the AE learns waveform representations that jointly optimize both functionalities, with a tunable parameter controlling the trade-off. Simulation results show that the proposed design outperforms conventional schemes in both communication reliability and sensing accuracy, especially under short blocklength and fading conditions.\end{abstract}


\IEEEpeerreviewmaketitle

\section{Introduction}
Wireless communications and sensing have both evolved significantly over the years, each with its own set of advancements in system architectures and algorithms. However, despite their similarities in signal processing techniques, channel modeling, and even certain aspects of system design, these two fields have largely operated in isolation from one another. Recently, there has been a noticeable shift in focus, primarily driven by the need to optimize spectrum usage, reduce costs, and save energy. This shift has led to co-design of an integrated sensing and communication (ISAC) system \cite{chiriyath2017radar,liu2020joint}. 

The introduction of ISAC systems opens a multitude of challenges in various applications, research avenues, and standardization initiatives. These include (1) exploring theoretical boundaries and trade-offs through information and signal processing theory~\cite{kobayashi2018joint,Yifeng2023fundamental}, (2) collaboratively designing waveforms and interference management strategies to balance effective communication and robust sensing performance~\cite{Fan2018toward, Shuangyang2022Novel, du2023probabilistic}, and (3) integrating both functionalities to enhance communication and sensing performances in practical applications~\cite{liu2020joint}.

Most of the ISAC signal designs can be categorized in three different approaches: communication-centric, sensing-centric, and joint designs~\cite{zhou2022integrated}. Both communication- and sensing-centric designs utilize existing waveforms from one task for the other, whereas the joint design aims to optimize the waveform for both functionalities.
In this paper, we consider the joint waveform design for ISAC.
The dual-function waveform design requires a delicate balance between communication and sensing capabilities. 

{\color{black}In particular, in order to convey information, the transmitter (TX) must make use of a collection (a codebook) of possible waveforms (codewords), corresponding to the different information messages. The main performance metrics are the rate (number of bits per channel use conveyed by the codebook), and the block error rate (BLER)~\cite{cover1999elements}. In contrast, sensing typically consists of estimating some parameter contained in the sensing channel (e.g., delay for ranging~\cite{kay1993fundamentals}), and it is best enabled by a predetermined waveform with suitable characteristics (e.g., very peaky autocorrelation with low sidelobes). 
ISAC signal design considering this potential mismatch remains largely an open problem and appears as one of the challenges to solve to enable 6G technology~\cite{Yifeng2023fundamental}. 

Recently, deep learning (DL) has emerged as a powerful tool for wireless communication system design, particularly for problems that are difficult to address using classical optimization methods. In the context of ISAC signal design, DL enables flexible optimization over complex and high-dimensional design spaces~\cite{temiz2025deep}.

{\color{black}A growing body of work has explored DL-based ISAC waveform and transceiver design under various system and channel models. These include model-free and model-based online learning for cooperative ISAC~\cite{Pulkkinen2023ModelFree, Pulkkinen2024ModelBased}, beamforming design~\cite{bazzi2023outage, Xu2024BeamformingISAC}, joint waveform and receiver (RX) design for MIMO systems~\cite{Kang2023MIMOISAC}, symbol-level precoding~\cite{Zheng2024SLPISAC,jiang2025slp}, RIS-assisted ISAC systems~\cite{jiang2025joint}, end-to-end learning frameworks~\cite{MateosRamos2022E2EISAC}, and unsupervised learning approaches~\cite{Liu2023DistributedISAC, Temiz2025UnsupervisedISAC}.

Within this line of research, our work focuses on short-length noncoherent ISAC code design, which introduces distinct challenges that are not addressed in the existing studies.}

{\color{black}In this paper, we consider ISAC signal design with short codeword lengths, defined in the range of $[10, 200]$ symbols. In this regime, classical coding theory based on asymptotic block length analysis is not directly applicable~\cite{Yifeng2023fundamental}, as performance is strongly influenced by finite-block length effects.

In conventional communication systems, a code design specifies a mapping from information messages to transmitted codewords, typically optimized using algebraic or probabilistic constructions. In contrast, in the short-block length regime, such analytical design tools are limited, especially when additional sensing constraints are imposed.

Furthermore, pilot-based channel estimation is often inefficient in short-block length transmissions due to the associated overhead. As a result, noncoherent detection~\cite[Chapter 4.5]{proakis2008digital} is commonly adopted, which further complicates code design.

To address these challenges, we adopt a data-driven approach, where the codebook (i.e., the mapping from messages to transmitted signal sequences) is learned using a neural network (NN). The NN jointly optimizes the waveform for transmission with respect to both communication and sensing objectives and allows flexible balancing between them.}

\subsection{Literature Review}
\subsubsection{Signal Design for Noncoherent Communications} 
Noncoherent detection and decoding for Rayleigh fading channels have been studied under both block and fast fading conditions. In the block fading regime, early approaches combined differential modulation with convolutional or turbo codes, enabling low-complexity decoding without explicit channel estimation \cite{hoeher1999turbo}. Iterative schemes that combine outer codes with noncoherent demodulation, such as in \cite{chen2003joint}, later improved performance significantly, achieving results close to the noncoherent capacity. More recently, pilot-free polar-coded systems \cite{yuan2021polar} have used cyclic redundancy check (CRC)-aided decoding to implicitly estimate the channel, achieving near-coherent performance.

In the fast fading regime, orthogonal signaling schemes such as M-ary frequency shift keying (FSK) or on-off keying (OOK) have been shown to be effective, especially when combined with turbo or low-density parity-check (LDPC) codes, as in \cite{valenti2005iterative}. Beyond differential modulation, several works have explored amplitude-based and peaky signaling strategies, motivated by their theoretical optimality in the low-signal-to-noise ratio (SNR) noncoherent regime \cite{Verdu2002spectral}. These methods exploit energy detection and sparse signaling to remain robust under severe channel uncertainty.

\subsubsection{Signal Design for Sensing} 
Signal design for sensing has been extensively studied in the context of active sensing systems and parameter estimation.
For example, the use of constant modulus waveforms for sensing has been widely considered. Such waveforms can effectively avoid the potential non-linear signal distortion due to the imperfect power amplifier and improve the energy efficiency. These waveforms are often derived by solving the corresponding optimization
problems with non-convex constant modulus constraint (CMC)~\cite{Maio2009design,Guolong2014mimo}. 
The sensing waveform design based on the signal covariance has also appeared in the literature~\cite{Stoica2007probing}, where the formulated optimization problem can be solved by using semi-definite quadratic programming (SQP) algorithm in polynomial time. 
Furthermore, the sensing waveform design based on the ambiguity function has also been studied. Specifically, these waveforms can be optimized to minimize the weighted integrated sidelobe level~\cite{tianyu2025sensing,he2012waveform} or the peak sidelobe level~\cite{Jing2019designing}, where various optimization methods, such as cyclic algorithms and gradient descent algorithms, have shown good performance.

\subsubsection{Communication Codeword Design based on Neural Encoders} 
The codeword design using NNs is not a new topic, especially for improving the communication performance. DL has been widely used in channel coding and modulation design in physical layer communication. 
An autoencoder (AE) model is commonly used for optimizing coded modulation in communication. AE is originally a compression model that extracts features from the input data and recovers the data with little loss from the compressed features. This could be applied to communication by adding the channel between the encoder and the decoder part and adjusting the dimension of the encoded data. 
The framework for this AE-based end-to-end communication has been well established in~\cite{jiang2019turbo, ye2020deep, RL-Paper}.
This topic has been called \textit{channel AE} and has been further researched for diverse channel models~\cite{uhlemann2020deep, xue2021end, zhang2022svd} and for low complexity implementations~\cite{jamali2022productae, gunlu2023concatenated, fritschek2025mingru}. 

\subsection{Contributions}
This paper investigates joint waveform design for ISAC with short block lengths through heuristic analysis and DL-based optimization. The main contributions are as follows:

1) We provide a concise ISAC problem formulation considering noncoherent detection for communications and multiple target estimation for sensing. Based on the formulation, the optimal communication decoding principle and the modified Cramér-Rao lower bound (MCRB) characterizing the optimal estimation performance are derived analytically.

2) Based on the derived communication and sensing analysis, we provide heuristic understandings for the optimal code design, focusing on representative limiting cases. Particularly, we reveal that the optimal communication codewords use signal amplitude carrying information, where codeword symbols should have different power distributed in order to improve the detection performance. Furthermore, we show that the optimal codeword for sensing also requires unevenly distributed power among codeword symbols, where the value of the MCRB does not change by {\color{black}a common phase rotation} of the codeword. According to these understandings, we unveil the structural similarities and potential trade-offs between the considered communication and sensing objectives. 

3) We develop a flexible framework for the considered ISAC waveform design based on the derived performance metrics, where a tuneable trade-off parameter is introduced to balance the communications and sensing performance. Our results demonstrate that the proposed method outperforms conventional coding and modulation schemes under short block lengths. More importantly, our results show a good alignment with our heuristic understandings based on the theoretical analysis.

\subsection{Notations} 
The superscript $(\cdot)^{\rm{H}}$ denotes the Hermitian of a matrix; $x^{*}$ denote the conjugate of $x$; ${\bf I}_N$ represents the identity matrix of size $N \times N$; $\mathbb{C}$ denotes the field of complex numbers; $\Pr(\cdot) $ denotes the probability of an event; ${\rm diag}(\bf x)$ outputs the diagonal matrix constructed by elements in $\bf x$; $||{\bf X}||$ denotes the $l_2$-norm of $\bf X$;
$\mathbb{E}[\cdot]$ denotes the statistical expectation; 
$\mathcal{CN}\left(\bm{\mu},\bm{\Sigma}^2 \right)$ denotes a circularly symmetric complex multivariate Gaussian distribution with mean $\bm{\mu}$ and covariance $\bm{\Sigma}^2$, respectively; $\nabla_{\bm\theta}$ is the gradient operator with respect to vector ${\bm \theta}$; Element-wise multiplication of matrices is denoted by $\odot$.

\begin{figure}[t!]
\centering
\includegraphics[width=\columnwidth]{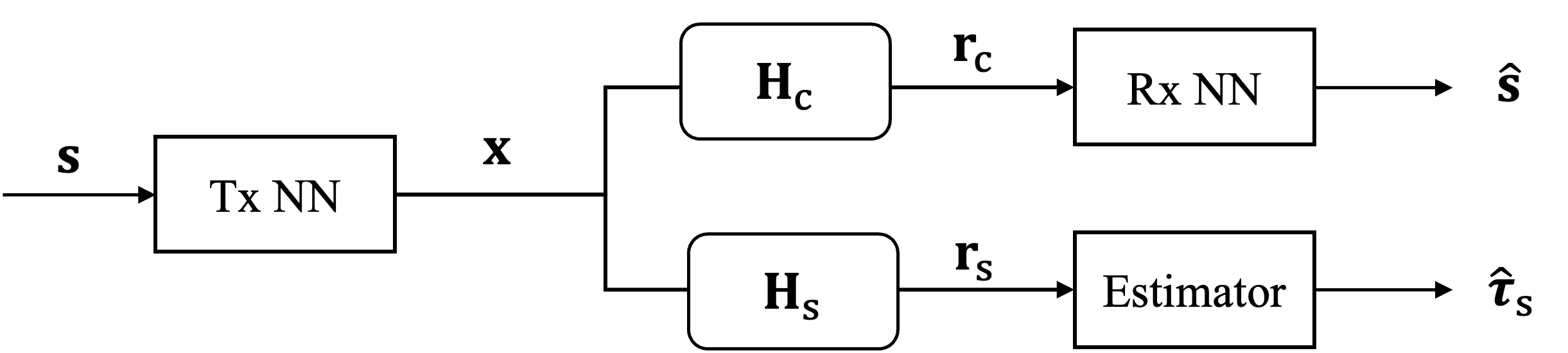} \caption{A diagram for the considered ISAC system based on the NN encoder.}
\label{System_model}
\centering
\end{figure}

\section{System Model}

We consider the ISAC transmission with a short block length, where the communication channel and sensing channel matrices are mutually independent and are denoted by ${\bf H}_{\rm c}\in \mathbb{C}^{N\times N}$ and ${\bf H}_{\rm s}\in \mathbb{C}^{N\times N}$, respectively.
Particularly, we focus on the noncoherent detection at the communication RX.
This noncoherent communication detector aligns with the nature of short block transmissions, operating without explicit knowledge of the communication channel coefficients to reduce the overhead of pilot signals for channel estimation. 

According to Fig.~\ref{System_model}, a $K$-bit information sequence $\bf s$ is first encoded using the proposed NN, resulting in a length-$N$ codeword of complex symbols $\mathbf{x}\in\mathbb{C}^{N\times 1}$ with code rate $K/N$ for transmission. After passing through the communication and sensing channels, the received symbol vectors for communication and sensing are given by ${\bf r}_{\rm c}$ and ${\bf r}_{\rm s}$, respectively, where 
\begin{align}
{{\bf{r}}_{\rm{c}}} = {{\bf{H}}_{\rm{c}}}{\bf{x}}+ {{\bf{w}}_{\rm{c}}}, \quad\text{and}\quad
{{\bf{r}}_{\rm{s}}} = {{\bf{H}}_{\rm{s}}}{\bf{x}}+ {{\bf{w}}_{\rm{s}}}.\label{eq:io_relation}
\end{align}
The additive noises ${\bf w}_{\rm c}$ and ${\bf w}_{\rm s}$ are defined as
${\bf w}_{\rm c} \sim \mathcal{CN}(0, N_{0,{\rm c}}{\bf I}_{N})$ and ${\bf w}_{\rm s} \sim \mathcal{CN}(0, N_{0,{\rm s}}{\bf I}_{N})$.
Both noise vectors ${{\bf{w}}_{\rm{c}}}$ and ${{\bf{w}}_{\rm{s}}}$ are statistically independent from ${{\bf{H}}_{\rm{c}}}{\bf{x}}$ and ${{\bf{H}}_{\rm{s}}}{\bf{x}}$, respectively.
The RX for communication attempts to recover the data sequence $\bf s$ based on ${{\bf{r}}_{\rm{c}}}$, and the resultant estimated bit vector is denoted by $\hat {\bf s}$. 
Meanwhile, the RX for sensing attempts to estimate the delays of $P$ paths based on ${{\bf{r}}_{\rm{s}}}$, and the resultant estimated delay vector is denoted by $\hat{\bm 
\tau}=(\tau^{(1)}, \tau^{(2)}, \dots, \tau^{(P)})$. 
It should be noted that the considered block diagram accounts for general ISAC transmissions with mono-static transceivers.
{\color{black}In this work, we focus on delay (ranging) estimation and do not explicitly model Doppler effects, which is a common assumption when targets are static or slowly varying over the observation interval. This simplification allows us to isolate the impact of waveform design on delay estimation and derive tractable performance metrics. Extending the framework to joint delay--Doppler estimation would require augmenting the parameter set and re-deriving the corresponding Fisher information matrix (FIM), which is left for future work.}

Consider $f_{\theta}\left( {\cdot} \right)$ and $g_{\theta}\left( {\cdot} \right)$ as an encoder and a decoder, respectively. The encoder performs a one-to-one mapping of the information bit sequence ${\bf s}$ to the power-normalized codeword $\bf x$. The decoder, on the other hand, takes the channel output and gives the estimated bit sequence. 
As such, our design objective is to find  effective $f_{\theta}\left( {\cdot} \right)$ and $g_{\theta}\left( {\cdot} \right)$ functions using NN-based methods to achieve good communication and  sensing performance. Here the subscript $\theta$ indicates that the function is parametrized and optimized using an NN.

\subsection{Communication Model}
We consider a frequency-flat Rayleigh fading communication channel {\color{black}for single-carrier communication}.
Specifically, we assume that ${{\bf{H}}_{\rm{c}}}$ is a diagonal matrix, i.e., ${{\bf{H}}_{\rm{c}}} = {\rm{diag}}\left( {{{\bf{h}}_{\rm{c}}}} \right)$, where ${{{\bf{h}}_{\rm{c}}}}$ is a length-$N$ vector composed of the communication channel coefficients at different time instants.
We further assume that ${{{\bf{h}}_{\rm{c}}}}\sim\mathcal{CN}(0, {\bf{R}}_c)$, where the elements are mutually correlated, characterized by covariance ${\bf{R}}_c$.

We apply the Jake and Clark's (JC's) time varying channel model for characterizing the correlated channel. In the JC model, the autocorrelation function is defined by the modified Bessel function of the first kind of order 0
$J_0(\cdot)$, such as
\begin{equation}
R_h(\ell) = J_0\left( 2\pi f_{\rm D} (\ell T_{\rm s}) \right), \label{eq_ACF}
\end{equation}
where $f_{\rm D}=(\nu/c)f_0$ is the maximum Doppler frequency, $\nu$ is the relative speed of the RX, $c$ is the speed of light, $f_0$ is the carrier frequency, $\ell T_{\rm s}$ is the channel symbol duration corresponding to  $\ell$ time instants, and 
\begin{equation} 
J_0(x)=\sum_{m=0}^\infty \frac{(-1)^m}{m!m!}\left(\frac{x}{2}\right)^{2m}=\frac{1}{\pi}\int_{0}^{\pi}\cosh(x\cos\theta)d\theta.
\end{equation}
For a given $f_{\rm D}$, the covariance ${{\bf{R}}_{\rm c}}$ of ${{{\bf{h}}_{\rm{c}}}}$ is defined as 
\begin{align}
{{\bf{R}}_{\rm c}} &= {\mathbb E}\left[ {{{\bf{h}}_{\rm{c}}}{\bf{h}}_{\rm{c}}^{\rm{H}}} \right]  \notag\\
&= \left[ {\begin{array}{*{20}{c}}
1 & R_h(1) & \cdots & R_h(N-1)\\
R_h(1) & 1 & \ddots & \vdots \\
 \vdots & \ddots & \ddots & R_h(1)\\
R_h(N-1) & \cdots & R_h(1) & 1
\end{array}} \right].
\end{align}

We highlight that the time-varying nature of the communication channel can be characterized by the above consideration, and as a matter of fact, both slow fading (coherence time longer than $N$ symbol periods) and fast fading (coherence time shorter  than one symbol period) represent two limiting cases of our channel model when the channels are fully correlated $R_h(\ell) \approx 1$ or uncorrelated $|R_h(\ell)|\approx 0$ for all $\ell=1,2,\cdots, N-1$, while partially correlated channels, i.e., $-1< R_h(\ell) <1$, correspond to intermediate fading conditions.
In what follows, we assume that $f_{\rm D}$ and $T_{\rm s}$ are known a priori for our design.\footnote{The code design depends on $f_\text{D}$, which reflects the user’s velocity. In practice, multiple codebooks can be pretrained for different $f_\text{D}$ values, and the appropriate one selected based on the measured $f_\text{D}$.
}

\subsection{Sensing Model}
The sensing RX aims to provide accurate ranging performance for targets. Therefore, we consider a single-carrier system for sensing $P$ targets. In this case, ${\bf H}_{\rm s}$ is a Toeplitz matrix, whose $\left(m,n\right)$-th entry is given by
\begin{align}
g_{m-n} \!=\! \sum_{i=1}^{P} h_{\rm s}^{(i)} \! \int_{ - \infty }^\infty {r}\left( t-nT_{\rm s}-\tau^{(i)}\right){r^*}\left( {t - mT_{\rm s}} \right){\rm{d}}t \label{rad_g_m_n}
\end{align}
The element is denoted by $g_{m-n}$ since its value depends on the difference $m-n$, not on the absolute values of $m$ and $n$.
In~\eqref{rad_g_m_n}, $r\left(t\right)$ is the underlying Nyquist-root shaping pulse applied at both the TX and RXs, and $T_{\rm s}$ is the Nyquist symbol time. Furthermore, $h_{\rm s}^{(i)}$ is the sensing fading coefficient for the $i$-th target, which is defined as $h_{\rm s}^{(i)} \sim {\cal CN}\left(0,1\right)$. The delay of each target is represented by $\tau^{(i)}$. Assuming that a coarse synchronization has already taken place after some forms of block boundary detection, the delays are therefore constrained within the frame duration, i.e., $0 \le \tau^{(i)} < NT_{\rm s}$ for all $i$. We consider uniformly distributed delays, i.e., $\tau^{(i)} \sim \mathcal{U}[0, L_{\rm s}T_{\rm s}]$ for some constant $0<L_{\rm s}< N$.

Based on the above description, the sensing RX attempts to estimate the delay $\tau^{(i)}$ of each target from ${{\bf{r}}_{\rm{s}}}$, and the corresponding estimate is denoted by ${\hat \tau}^{(i)}$. {\color{black}
It is assumed that the number of targets $P$ and the delay statistics are known during the code design phase, as part of the adopted sensing channel model. This assumption is consistent with the literature that considers detection (deciding the number of targets) and estimation (deciding the values of parameters) separately~\cite{kay1993fundamentals}. 

This study focuses on optimizing the waveform for the aforementioned delay estimation problem for a given and fixed $P$.
Extending the framework to scenarios with unknown or time-varying $P$ is left for future work.}

{\color{black}\section{Theoretical Analysis }\label{sec:optimal_schemes}
This section presents the theoretical analysis underlying the proposed framework, including derivation of the noncoherent maximum likelihood (ML) RX for communication and the MCRB for sensing. These results provide theoretical justification for the adopted performance metrics, which are further used to guide the training of the DL-based waveform design.

\subsection{Maximum Likelihood Detection for Noncoherent Communication}}
The noncoherent detection/decoding problem is not new, and several insightful results are well-established  in communication theory. 
However, most noncoherent detectors, e.g.,~\cite{proakis2008digital}, only consider symbol-wise detection, which is not directly applicable to our case where the channel coefficients can vary over time. In what follows, we derive the frame-wise optimal detector of interest for decision-making.

Since the information messages are equiprobable, it is well-known that the BLER is minimized by the ML decoder as characterized below:

\begin{equation}\label{ML_rule}
{\bf{\hat x}} = \mathop {\arg \max }\limits_{{\bf{x}} \in {\cal X}} \ \Pr \left({\bf{r}}_{\rm{c}} |{{\bf{x}}} \right).
\end{equation}
Then, the information vector $\hat{\bf{s}}$ is obtained by one-to-one remapping from $\hat{\bf{x}}$.
The likelihood $\Pr \left({\bf{r}}_{\rm{c}} |{\bf{x}}\right)$ is the transition probability of the channel ${{\bf{r}}_{\rm{c}}} = {{\bf{H}}_{\rm{c}}}{\bf{x}}+ {{\bf{w}}_{\rm{c}}}$, where ${{\bf{r}}_{\rm{c}}}$ is conditionally circularly symmetric Gaussian for given ${\bf{x}}$, via the covariance ${\bf{R}}_{{\bf{r}_{\rm c}|{\bf{x}}}}=N_{0, \rm c}\mathbf{I}_N+\mathbf{R}_{\rm c}\odot{{\bf{x}}{{\bf{x}}^{\rm{H}}}} =
N_{0, \rm c}\mathbf{I}_N+{{\bf{X}}\mathbf{R}_{\rm c}{{\bf{X}}^{\rm{H}}}}$, where ${\bf{X}}=\text{diag}({\bf{x}})$. This leads to
\begin{equation}\label{ML_rule_prob}
\Pr \left( {{{\bf{r}}_{\rm{c}}}\left| {{\bf{ x}}} \right.} \right) = \frac{1}{ \pi^N{\det \left( {\bf{R}}_{{\bf{r}_{\rm c}|{\bf{x}}}} \right) } }\exp \left( { - {\bf{r}}_{\rm{c}}^{\rm{H}}{\bf{R}}_{{\bf{r}_{\rm c}|{\bf{x}}}}^{ - 1}{{\bf{r}}_{\rm{c}}}} \right).
\end{equation}
Finally, according to~\eqref{ML_rule}, ${\bf{\hat s}}$ is the one that maximizes the probability in~\eqref{ML_rule_prob}, i.e. ${\bf{\hat s}}=\argmax_{{\bf{ s}}\in\{0,1\}^{K} }\Pr \left( {{{\bf{r}}_{\rm{c}}}\left| {{\bf{ x}}} \right.}=f_{\theta}({\bf{ s}}) \right)$.

The derived ML decoding rule provides important insights for the codeword design, which will be presented in Section~\ref{sec:heuristic_understanding}.
While ML decoding provides optimal performance in terms of BLER, using likelihood as a performance metric is not scalable.
Evaluating communication quality or using it for training the NN requires computing the likelihood of all $2^K$ hypotheses. Therefore, we use bit error rate (BER) for evaluation and binary cross-entropy (BCE) loss for training, as will be explained in detail in Section~\ref{sec:DL-based_design}.

\subsection{Cram\'er-Rao Lower Bounds Analysis for Sensing}
The Cram\'er-Rao Lower Bound (CRLB) provides a fundamental limit on estimation error of unbiased estimators, stating that their precision is bounded by the inverse of the FIM, as established in~\cite{van2004detection}. 
Suppose that $\bm{u}=(u_1, u_2, \dots, u_d)\in\mathbb{R}^d$ is the parameter vector to be estimated, and that $\bf{y}$ is the observation defined as the conditional distribution {\color{black}$p_{\bf{y}}({\bf{y}}|\bm{u})$} for a given $\bm{u}$.
Then, 
{\color{black}
for any unbiased estimator ${\hat{\bm{u}}(y)}$ of ${\bm{u}}$, the estimator covariance satisfies 
\begin{equation}
\text{Cov}[\hat{\bm{u}}({\bf y})]\succeq {\bf{J}}^{\rm -1}(\bm{u}), \label{eq_CRLB}
\end{equation}
where ${\bf{J}}(\bm{u})$ is the FIM defined by its $(i,j)$-th element  
\begin{equation}
    [{\bf{J}}(\bm{u})]_{i,j}=-\mathbb{E}_{{\bf{y}}|\bm{u}}\left[\frac{\partial^{2}\log(p_{\bf{y}}({\bf{y}}|\bm{u}))}{\partial u_i\partial u_j}\right],\label{eq_CRLB_FIM}
\end{equation}
for $i,j=1,2,\dots,d$.

When observation ${\bf y}$ depends not only on ${\bm u}$, but also nuisance parameters ${\bm v}=(v_1, v_2, \dots, v_{d'})$, the conditional probability density function (PDF) can be obtained through $p_{\bf{y}}({\bf{y}}|{\bm u}) =\int_{\bm v}p_{\bf{y}}({\bf y}|{\bm u, \bm v})p_{\bm v}(\bm v|\bm u)d\bm v$, and the CRLB can be obtained by using \eqref{eq_CRLB}. The CRLB holds for an estimator of $\bm u$ that is unbiased for $\bm v$ in the global sense. 

In case the computation of $p_{\bf{y}}({\bf{y}}|{\bm u})$ is complicated, the Miller-Chang lower bound \cite{miller1978modified} can be useful. The Miller-Chang computes the CRLB for each realization of nuisance parameters $\bm{v}$ and then takes the average of it: 
\begin{align}
&\text{Cov}[\hat{\bm{u}}({\bm y})]\geq \mathbb{E}_{\bm v}[{\bf J}^{-1}(\bm u,\bm{v})]=\int_{\bm v} {\bf J}^{-1}(\bm u,\bm{v})p_{\bm v}(\bm v)d{\bm v},\\
&\text{where } [{\bf J}(\bm u,\bm{v})]_{i,j}=-\mathbb{E}_{{\bf y}|\bm{u}, \bm v}\left[\frac{\partial^{2}\log(p_{\bf{y
}}({\bf y}|\bm{u}, \bm v))}{\partial u_i\partial u_j}\right],
\end{align}
for $i,j=1,2,\dots,d$. This holds for an estimator that is unbiased for every value of $\bm v$.

For multi-dimensional problems where $d>1$ and $d'>1$, solving an analytical form of the inverse of the FIM can be challenging. A modified Cramér-Rao lower bound (MCRB) from \cite{gini1998modified} can be useful for such cases, which states that for any unbiased estimator of $\bm u$, it holds
\begin{align}
&\text{Cov}[\hat{\bm{u}}({\bm y})] \succeq \left( \mathbb{E}_{\bm v}[{\bf J}(\bm u,\bm{v})] \right)^{-1}.
\end{align}
The MCRB is less tight than the Miller-Chang lower bound, however, it simplifies the FIM before taking inverse. 

For the ranging problem considered, the parameters to estimate is the delay vector $\bm \tau$, while the random fading ${\bf h}_{\rm s} = (h^{(1)}_{\rm s}, h^{(2)}_{\rm s},\dots, h^{(P)}_{\rm s})$ and the random input $\bf x$ are regarded as nuisance parameters. 
It can be easily shown that in our case, the standard CRB for given $\bf x$, obtained by first removing the conditioning on $\bf h_{\rm s}$ and then applying the FIM, yields a non-diagonal FIM expression whose inverse cannot be given in closed form in general.
This motivates the use of the MCRB as an analytically tractable metric for the sensing performance.
For each delay $\tau^{(i)}$, the bound of the estimation error can be written:
\begin{align}
\text{Var}[\hat{\tau}^{(i)}|{\bf x}] \geq&\ [ \left( \mathbb{E}_{{\bf h_{\rm s}}\!}\left[ {\bf J}(\bm \tau,{\bf h}_{\rm s}, {\bf x})\right]\right)^{-1} ]_{i,i}, \label{eq_MCRB_}\\
\hspace{-0.2cm} [{\bf J}(\bm \tau, {\bf h}_{\rm s}, {\bf x})]_{i,j} \!=\! 
-&\mathbb{E}_{{\bf y_{\rm s}}|\bm{\tau}, {\bf h}_{\rm s}, {\bf x} \! }\left[\frac{\partial^{2}\log(p_{\bf{y}_{\rm s}}({\bf y}_{\rm s}|\bm{\tau}, {\bf h}_{\rm s}, {\bf x}))}{\partial\tau_i\partial\tau_j}\right]\label{eq_MCRB_element}
\end{align}}

{\color{black}In this study, we consider an energy-normalized Nyquist root pulse $R(t)$ such that $R(f) = \sqrt{P(f)}$ for a Nyquist pulse $P(f)$ and $\int_{\infty}^{\infty}|r(t)|^2dt=1$.}
The corresponding MCRB is provided in the following theorem. 

\begin{theorem}\label{thm_MCRB}{\color{black}
For any $P$ and any delay $\tau^{(a)}$ for $a=1,2,\dots, P$, the MCRB \eqref{eq_MCRB_} takes on the expression \begin{align}\label{eq_MCRB}
&\text{Var}[\hat{\tau}^{(a)}|{\bf x}]\\
\quad \geq&\ \frac{N_{0,\rm s}}{2} \left(\!\sum_{m=1}^N \left|\sum_{n=1}^N x_n 
{\color{black}p'\left((m-n)T_{\rm s}-\tau^{(a)} \right)}
\right|^2\right)^{\!-1},\!\!
\end{align}
where $p'(x) \coloneq \frac{\partial }{\partial x}p(x)$. We note that the non-diagonal element where $i\neq j$ in~\eqref{eq_MCRB_element} is zero, which results in a diagonal FIM.}
\end{theorem}
\proof
Refer to Appendix~\ref{proof_thm_MCRB}.

{\color{black}We denote the lower bound by $\text{MCRB}({\bf x}, \tau^{(a)})$.} Remark that the MCRB is a diagonal matrix under this estimation problem and can be computed element-wise.

\section{Design Principles for Code Optimization}\label{sec:heuristic_understanding}
In this section, we discuss code design principles based on the theoretical analysis developed in the previous section. Note that the optimal code design for ISAC is still an open problem, especially in the short block length regime. Therefore, we are here showing heuristic understandings of the problem by considering special examples, which serve as an important building block for our performance evaluation.

\subsection{Implications of the Likelihood on Code Design for Noncoherent Detection with Non- and Fully-Correlated Coefficients}

As a driver for the proposed DNN-based code design, we consider two limiting cases with uncorrelated and fully correlated coefficients.
For uncorrelated coefficients, the covariance ${\bf R}_{\rm c}$ of the channel coefficients is an identity matrix, while for fully correlated coefficients, the covariance is a matrix of ones. 
In both cases, the conditional covariance 
${\bf{R}}_{{\bf{r}}_{\rm c}|{\bf{x}}}$ and the likelihood function in~\eqref{ML_rule_prob} can be further simplified as shown in Appendix~\ref{appendix:likelihood_specialcase}. 

For uncorrelated coefficients, the likelihood function is
\begin{align}\label{ML_rule_non_correlated}
\Pr \left( {{\bf{r}}_{\rm{c}}\left| {{\bf{ x}}} \right.} \right) 
= \frac{\exp \left( - \sum_{i = 1}^N \frac{|r_{{\rm c},i}|^2}{|x_i|^2+N_{0,\rm c}} \right)}{{
\pi^N\prod_{i=1}^N{\left(|x_i|^2 + N_{0,\rm c} \right)}
}}, 
\end{align}
where ${r'_{{\rm{c}},i}}$ is the $i$-th element of ${\bf{r}'}_{\!\!\rm c}$. 
Since~\eqref{ML_rule_non_correlated} depends only on the symbol amplitudes $|x_i|^2$, PSK-type signals perform poorly as all symbols have identical energy and indistinguishable.

The optimal code design can be further studied through the pairwise-error probability (PEP). A pairwise error occurs if a competing codeword $\mathbf{x}'$ yields a higher likelihood than the true codeword $\mathbf{x}$, i.e., when
\begin{equation}
\ln\frac{\Pr({\bf r|x})}{\Pr({\bf r|x'})}=-\sum_{i=1}^N\frac{|r_{{\rm c},i}|^2}{|x_i|^2+N_{0, \rm c}} +\sum_{i=1}^N\frac{|r_{{\rm c},i}|^2}{|x'_i|^2+N_{0, \rm c}} < 0.
\end{equation}
Let $\mathcal{S}({\bf h}_{\rm c})$ be the set of indices where $|h_i|\not\ll 1$. For indices with $|h_i|\ll1$, the summands essentially cancel, so the log-likelihood (LLR) simplifies to
\begin{equation}
-\underbrace{\sum_{i\in\mathcal{S}({\bf h_{\rm c}})}\frac{|h_{{\rm c},i}x_i+w_{{\rm c}, i}|^2}{|x_i|^2+N_{0, \rm c}}}_{A} +\underbrace{\sum_{i\in \mathcal{S}({\bf h_{\rm c}})}\frac{|h_{{\rm c},i}x_i+w_{{\rm c}, i}|^2}{|x'_i|^2+N_{0, \rm c}}}_{B}.
\end{equation} 
{\color{black}In order to gain insight on the qualitative behavior of the pairwise error metric and therefore obtain criteria for the code optimization, we consider these terms in the case where the instantaneous signal strength is either much larger or much smaller than the noise level $N_{0, \rm c}$.}
Denote the approximation by $-A+B=-\sum_{i\in\mathcal{S}({\bf h}_{\rm c})}A_i + \sum_{i\in\mathcal{S}({\bf h}_{\rm c})}B_i$. Let $\mathcal{L}(\mathbf{x})$ denote the set of indices with $|x_i|^2 \gg N_{0,\rm c}$. The summands can then be approximated as
\begin{align}
A_i &\approx 
\begin{cases}
    |h_{{\rm c}, i}|^2, & |x_i|^2\gg N_{0, \rm c} \\
    \tfrac{|w_{{\rm c}, i}|^2}{N_{0, \rm c}}, & |x_i|^2\ll N_{0,\rm c},
\end{cases} \\
B_i &\approx
\begin{cases}
    |h_{{\rm c}, i}|^2, & |x_i|^2 \gg N_{0, \rm c},\ |x'_i|^2 \gg N_{0, \rm c} \\[2pt]
    \tfrac{|w_{{\rm c}, i}|^2}{N_{0, \rm c}}, & |x_i|^2 \ll N_{0, \rm c},\ |x'_i|^2 \ll N_{0, \rm c} \\[2pt]
    0, & |x_i|^2 \ll N_{0, \rm c},\ |x'_i|^2 \gg N_{0, \rm c} \\[2pt]
    \tfrac{|h_{{\rm c}, i}x_i|^2}{N_{0, \rm c}}, & |x_i|^2 \gg N_{0, \rm c},\ |x'_i|^2 \ll N_{0, \rm c}.
\end{cases}
\end{align}
Let $\mathcal{S}=\mathcal{S}({\bf h}_{\rm c})$, 
$\mathcal{L}=\mathcal{L}(\mathbf{x})$, 
and $\mathcal{L}'=\mathcal{L}(\mathbf{x}')$. 
Then, we have
\begin{align}
A &\approx \sum_{i\in\mathcal{S}\cap\mathcal{L}} |h_{{\rm c}, i}|^2 
    + \sum_{i\in\mathcal{S}\setminus\mathcal{L}} \!\!\! \tfrac{|w_{{\rm c}, i}|^2}{N_{0, \rm c}}, \\[4pt]
B &\approx \!\!\sum_{i\in\mathcal{S}\cap\mathcal{L}\cap\mathcal{L}'} \!\!|h_{{\rm c}, i}|^2\! 
    + \!\!\!\sum_{i\in\mathcal{S}\cap\mathcal{L}\setminus\mathcal{L}'} \!\!\! \tfrac{|h_{{\rm c}, i}x'_i|^2}{N_{0, \rm c}}
    + \!\!\!\sum_{i\in\mathcal{S}\setminus(\mathcal{L}\cup\mathcal{L}')} \!\!\!\!\! \tfrac{|w_{{\rm c}, i}|^2}{N_{0, \rm c}}.
\end{align}

From this analysis, we conclude that the PEP is small with high probability when each codeword contains both large and small components, so that a few deep fades cannot erase the entire signal. Moreover, among the surviving terms $i \in \mathcal{S}({\bf h}_{\rm c})$, the PEP decreases if the large components of two codewords do not coincide. Designing codewords with little overlap in their strong entries ensures that many indices are strong for $\mathbf{x}$ but weak for $\mathbf{x}'$, yielding net positive separation in the LLR and lowering the PEP. If the large components overlap, however, their contributions cancel, reducing distinguishability.

Next, the likelihood function for fully-correlated coefficients can be similarly derived as follows:
\begin{align}
&\Pr \left({{{\bf{r}}_{\rm{c}}}\left| {{\bf{ x}}} \right.} \right)= \frac{\exp \Big( -\frac{|{\bf{r}}_{\rm c}|^2}{N_{0,\rm c}} \!+\!\frac{|{\bf{r}}_{\rm c}^{\rm H}{\bf{x}}|^2}{N_{0,\rm c}^2\big(1+\frac{|{\bf{x}}|^2}{N_{0,\rm c}}\big)}\Big)}{{
\pi^N\!N_{0,\rm c}^N\Big(1\!+\!\frac{|{\bf{x}}|^2}{N_{0,\rm c}}\Big)}} .\label{eq_likelihood_lb}
\end{align}
When the signal power $|{\bf{x}}|$ is constant regardless of codewords, the likelihood only depends on the amplitude of the inner product
$|{{\bf{r}}}_{\rm c}^{\rm H}{\bf{x}}|^2=|{\bf{r}}_{\rm c}\cdot{\bf{x}}|^2$. Hence, the ML decision can be simplified as the \textit{square-law} detector
    $\hat{\bf{x}}=\arg\max_{\bf{x}'} |{\bf{r}}_{\rm c}\cdot{\bf{x}'}|^2.$
The optimal constellation maximizes the amplitude of the inner product of the correct hypothesis, while suppressing that of the others on the $N$-dimensional sphere.
{\color{black}The expected squared correlation between ${\bf{r}}_{\rm c}$ and a generic codeword $\bf{x}'$, given that $\bf{x}$ is transmitted, takes on the form}
 \begin{align}
    &\mathbb{E}_{{\bf{r}}_{\rm c}|{\bf{x}}}\left[|{\bf{r}}_{\rm c} \!\cdot\! {\bf{x}'}|^2\right]
    =\mathbb{E}_{{\bf{H}}_{\rm{c}}}\!\left[|({{\bf{H}}_{\rm{c}}}{\bf{x}}) \!\cdot\! {\bf{x}'}|^2\right] \!+ \mathbb{E}_{{{\bf{w}}_{\rm{c}}}}\!\left[|{{\bf{w}}_{\rm{c}}} \!\cdot{\bf{x}'}|^2\right]\\
    &=\ \mathbb{E}_{{\bf{H}}_{\rm{c}}}
    \!\Big[|{{{h}}_{{\rm{c}},1}}|^2\sum_{i=1}^N |{{x}^{*}_i{x'}_i}|^2\Big] 
    \!+ N_{0,\rm c}\sum_{i=1}^N| {{x}'_i}|^2 \\
    &=\sum_{i=1}^N |{{x}^{*}_i{x}'_i}|^2 + N_{0,\rm c}|{\bf{x}'}|^2 = |{\bf{x}'}^{\rm H}{\bf{x}}|^2+N_{0,\rm c}|{\bf{x}'}|^2.
\end{align}
Note that expected LLR depends only on mutual coherence $|{\bf{x}}'^{\rm H}{\bf{x}}|^2$ when the signal power $|{\bf{x}}'|^2$ is constant. 
This suggests that good spherical codes must have low maximum squared correlation between codewords.
A classical lower bound for mutual coherence between codewords in a codebook of size $2^K$ and dimensions $N$ is given by the Welch bound~\cite{welch1974lower}, and codebooks achieving the Welch bound are referred to as Welch bound equality sets, or equiangular tight frames.
For particular pairs of $N$ and $K$, such a set can be analytically solved, otherwise it requires numerical solutions. 

{\color{black}We have two limiting cases, and the intuition gained from the pairwise error probability analysis and expected LLR yield different design guidelines: 
for independent fading, codewords should have variable-amplitude symbols and the position of their strong components should be as distinct as possible. 
For correlated fading, the low-cross correlation code design principle applies. 
We conjecture that a mixture of these two characteristics yields good codes in the intermediate correlation conditions. This is difficult to design analytically, and we resort to the learning-based optimization in Sections \ref{sec:sim_toy_example} and \ref{sec:sim_longer_blocklength}.}

\subsection{Implications of MCRB on Code Design for Sensing} \label{ssec:heuristic_understanding_sensing}
Theorem~\ref{thm_MCRB} shows how MCRB is determined for given $\bm\tau$ and ${\bf x}$. After some manipulations, we can rewrite~\eqref{eq_MCRB} by  
\begin{align}
&MCRB({\bf x}, \tau^{(a)}) =\frac{N_{0,\rm s}}{2}\left({\bf x}^{\rm H}{\bf V}{\bf x}\right)^{-1},\label{eq_MCRB_derivation}
\end{align}
where ${\bf V}$ is defined by the $\left(l,k\right)$-th element equal to $\sum_{m=1}^N {\color{black}p'\left((m-k)T_{\rm s}-\tau^{(a)} \right)}{\color{black}p'\left((m-\ell)T_{\rm s}-\tau^{(a)} \right)}$. The optimal sequence of $\bf x$ that minimizes $MCRB({\bf x}, \tau^{(a)})$ should maximize ${\bf x}^{\rm H}{\bf V}{\bf x}$, which, under the power constraint, is the solution of the maximum eigenvector problem. Based on this, we can obtain the following two insights for code design:
\begin{itemize}
    \item Given $\tau^{(a)}$, the optimal sequence of $\bf x$ can be obtained by the principal eigenvector of $\bf V$. The eigenvector reduces to a constant vector only when all row vectors have identical row sums. Since this condition is generally not satisfied, the optimal sequence $\bf x$ will, in general, allocate different symbol energies across its entries.
    \item Notice that~\eqref{eq_MCRB_derivation} is a quadratic form. Therefore, applying {\color{black}a common phase rotation} to $\bf x$ does not change the MCRB value. This implies that the optimal sequence is not unique. 
\end{itemize}

Analytically solving the optimal sequence of $\bf x$ would be challenging since the value of $\tau^{(a)}$ cannot be known a prior.  
However, as we will demonstrate in the later part of this paper, the proposed NN can effectively find a good solution for the sensing waveform design without the knowledge of $\tau^{(a)}$. 
\subsection{Comparison of the Code Design for Communication and Sensing}
The code designs based on likelihood and MCRB exhibit both similarities and differences in the optimal strategies for communication and sensing. In both cases, the optimal codewords allocate power unevenly across symbols, indicating that constant-amplitude designs are suboptimal. However, the results also reveal that the optimal designs for communication and sensing may not always align, as the relationships among the optimal codewords can differ.

The MCRB is invariant to common phase rotation of $\bf x$, meaning that if one optimal vector exists, then all its rotated versions are also optimal. Likewise, the likelihood expressions in \eqref{ML_rule_non_correlated} and \eqref{eq_likelihood_lb} are invariant to a common phase rotation, and in the fully correlated case. 
In contrast to sensing, where rotation of an optimal vector remains optimal, communication requires codewords to be distinguishable; codewords that produce the same likelihood must be avoided to ensure reliability. This difference suggests a potential trade-off between communication and sensing. 

However, such a trade-off is not guaranteed because the set of sensing-optimal codewords may consist of multiple distinct sets, each spanned by rotation of a different vector, and some of these sets could still yield good likelihood performance for communication. {\color{black}Furthermore, it is difficult to deduce insight into code design for a general communication channel case with $f_{\rm D}\in (0,1)$ and its relation with the sensing optimal code.
To further clarify this uncertainty, we resort to numerical solution optimized by using DL. The following section elaborates on the design of the DL system for code optimization.}

\begin{algorithm}[t!]
\caption{The training algorithm of the framework.}\label{alg:training}
\begin{algorithmic}
\State Trainable parameters: weights of the AE ${\bm \theta}_{AE}$. 
\State Hyper parameters: learning rate $\eta_e$, loss ratio $\gamma$, batch size $B$, $\rm SNR_{low}$ and $\rm SNR_{ high}$ for training.
$\ell_{\rm c}$ and $\ell_{\rm s}$.
\For {Epochs} 
\For {Batches} 
\State {\bf Forward pass}
\For {$B$ samples indexed by $b=1,2,\dots,B$}
\State $\mathbf{s}^{(b)} \leftarrow$ Sample uniformly from $\{0,1\}^{\rm K}$ 
\State $\mathbf{x}^{(b)} = f_{\theta}(\mathbf{s}^{(b)})$ \hfill \textbf{Encoding}
\State Sample ${\rm SNR}^{(b)} \overset{\text{i.i.d.}}{\sim} \mathcal{U}_{[\rm SNR_{low}, SNR_{high}]}$
\State ${\rm SNR}^{(b)}_{\rm lin} \!=\!10^{{\rm SNR}^{(b)}/10}$, ${N^{(b)}_{0, \rm c}} \!=\! {N^{(b)}_{0, \rm s}} \!=\! 1/{\rm SNR_{lin}^{(b)}}$ 
\State $\mathbf{r}^{(b)}_{\rm c} = \mathbf{H}_{\rm c}^{(b)}\mathbf{x}^{(b)} + \mathbf{w}^{(b)}_{\rm c}$ \hfill \textbf{Channel}
\State $\mathbf{r}^{(b)}_{\rm s} = \mathbf{H}^{(b)}_{\rm s}\mathbf{x}^{(b)} + \mathbf{w}^{(b)}_{\rm s}$
\State $ {\bf B}^{(b)}  \leftarrow g_{\theta}(\mathbf{r}^{(b)}_{\rm s})$ \hfill \textbf{Decoding}
\EndFor
\State $\ell_{\rm c} \leftarrow \frac{1}{B}\sum_{b=1}^B\ell_{\rm c}({\bf s}^{(b)},{\bf B}^{(b)})$ \hfill \textbf{Communication loss}
\State $\ell_{\rm s} \leftarrow$ $\frac{1}{B}\sum_{b=1}^B \ell_{\rm s}({\bf x}^{(b)}, {\bm \tau}^{(b)} )$ \hfill \textbf{Sensing loss }
\State $\ell_j\leftarrow \gamma\ell_{\rm c} + (1-\gamma)\ell_{\rm s}$ \hfill \textbf{Joint loss}
\State \textbf{Backpropagation}
\State ${\bm \theta}_{AE} \leftarrow {\bm \theta}_{AE} - \eta_e$Opt($\ell_j$) \hfill \textbf{Update Autoencoder}
\EndFor 
\EndFor
\end{algorithmic}
\end{algorithm}

\section{Deep Learning-based Code Design}\label{sec:DL-based_design} 
To jointly train the TX NN for communication and sensing, we use a weighted sum of communication and sensing loss functions, denoted by $\ell_{\rm c}$ and $\ell_{\rm s}$, respectively:
$\ell_j \coloneqq \gamma \ell_{\rm c} + (1-\gamma) \ell_{\rm s}$, where $\gamma\in[0,1]$ is the weight of the communication loss. When $\gamma=1$, the TX NN is trained only to reduce the communication loss, and similarly when $\gamma=0$, it learns to reduce the sensing loss only.
For communication, we use BCE loss between the message ${\bf s}$ and the soft-decoded output of the RX NN because it works well for classification tasks. Let the soft-decoded output vector denoted by ${\bf B}({\bf r}_{\rm c})\coloneq(B_1, B_2, \dots, B_K)=f_{\theta}({\bf r_{\rm c}})\in[0,1]^K$. The BCE of one sample pair ${\bf s}$ and ${\bf B}$ is defined as 
\begin{equation}
\ell_{\rm c}({\bf s}, {\bf B})\!=\!-\frac{1}{K} \!\sum_{k=1}^K\left(s_k\log B_k \!+\!(1-s_k)\log(1-B_k)\right).\!
\end{equation}
The final decoded output is obtained simply rounding up the scores, i.e., $\hat{\bf s}=\lfloor{\bf B}\rceil$.
For sensing, we use the MCRB in~\eqref{eq_MCRB} averaged over $P$ targets without considering a particular delay estimator. 
The loss function of each input bit sequence ${\bf s}$ and its codeword ${\bf x}$ is defined as
\begin{align}
\ell_{\rm s}({\bf x}, {\bm \tau)}=&\frac{1}{P} \!\sum_{p=1}^P \frac{1}{T^2_{\rm s}}\text{MCRB}( {\bf x}, {\tau^{(p)}}).
\end{align}
The coefficient $1/T^2_{\rm s}$ is multiplied to normalize the MCRB with respect to the symbol time.
We assume genie training data, i.e. messages ${\bf s}$ are known at RX NN, and the number of targets $P$ and the delays are known at TX NN during the training phase.
During the inference phase, the RX NN only requires the channel output {\color{black}for decoding the message. For sensing, we assume that $P$ is estimated accurately by a target detection algorithm before delay estimation.} 
For training, we use a fixed value of $P$.
To ensure generalization across SNR values, we train the AE using uniformly random SNRs sampled within a predefined range, adjusting $N_{0, \rm c}$ and $N_{0, \rm s}$ accordingly. In the results of this paper, we used $N_{0, \rm c}=N_{0, \rm s}$ for simplicity, however, they can be configured independently.

The model can be trained in two modes: (1) joint training, where the weights of both the RX and TX NNs are updated simultaneously in each iteration; and (2) iterative training, where only one network, either the RX NN or TX NN, is updated at a time. Based on our observations, iterative training converges more stably and ultimately achieves better performance than joint training, consistent with the findings in~\cite{jiang2019turbo}.
One cycle of iterative training consists of five epochs of the RX NN followed by one epoch of the TX NN. Each epoch trains the network with 100 batches, each containing 500 randomly generated bit sequences. The learning rate is initially set to $10^{-4}$ and gradually reduced to $10^{-5}$. For the sensing-optimal solution, only the TX NN is trained.

{\color{black}In our simulations, we consider three AE architectures from~\cite{jiang2019turbo}: a convolutional NN (CNN) and two turbo-inspired variants (Turbo-CNN and Turbo-RNN), mainly to assess the impact of iterative structures and NN types. Implementation details are referred to~\cite{jiang2019turbo} and our code repository\footnote{\texttt{https://github.com/muahkim/ISAC-ML-Toeplitz-Clarkes}}.

All models operate at a real-valued code rate of $1/3$. For complex-valued codewords, real and imaginary parts are mapped to alternating entries, yielding an effective code rate of $2/3$. The encoder output is normalized as $\sum_{i=1}^N |x_i|^2 = N$.

The CNN-based AE is used as the primary model due to its simplicity and strong performance. It employs a single feed-forward convolutional branch for both encoding and decoding, without interleaving or iterative processing.

For comparison, Turbo-CNN and Turbo-RNN extend this structure using turbo-inspired parallel branches and iterative decoding. Turbo-CNN uses convolutional layers, while Turbo-RNN replaces them with recurrent units (e.g., GRUs), increasing complexity due to iterative and/or recurrent processing.}

\begin{figure*}[h!]
\centering
    \subfloat[\color{black}BERs of CNN trained for communication, compared with classical schemes.\label{fig:BER_K6N9_gamma1}]{\includegraphics[width=0.9\textwidth]{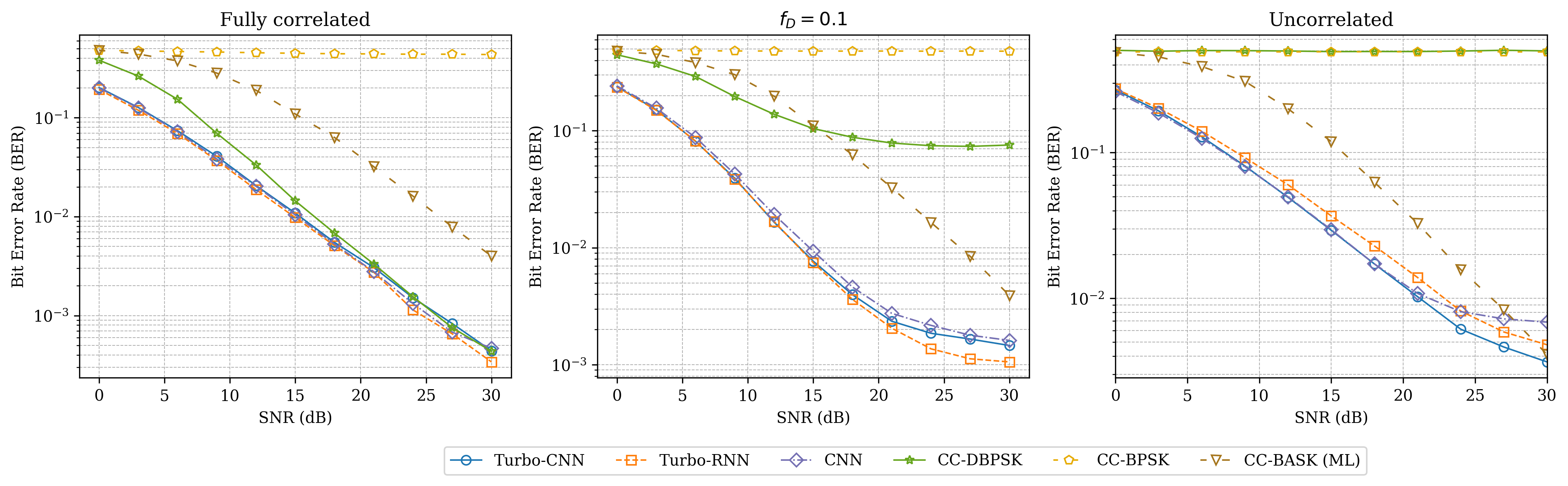}}
\\ 
    \subfloat[\color{black}Robustness to mismatched channel correlation. Each curve corresponds to a model trained at a specific $f_{\rm D}$ value, as indicated in the legend.\label{Fig_robustness_test}]
    {\centering \includegraphics[width=0.3\textwidth]{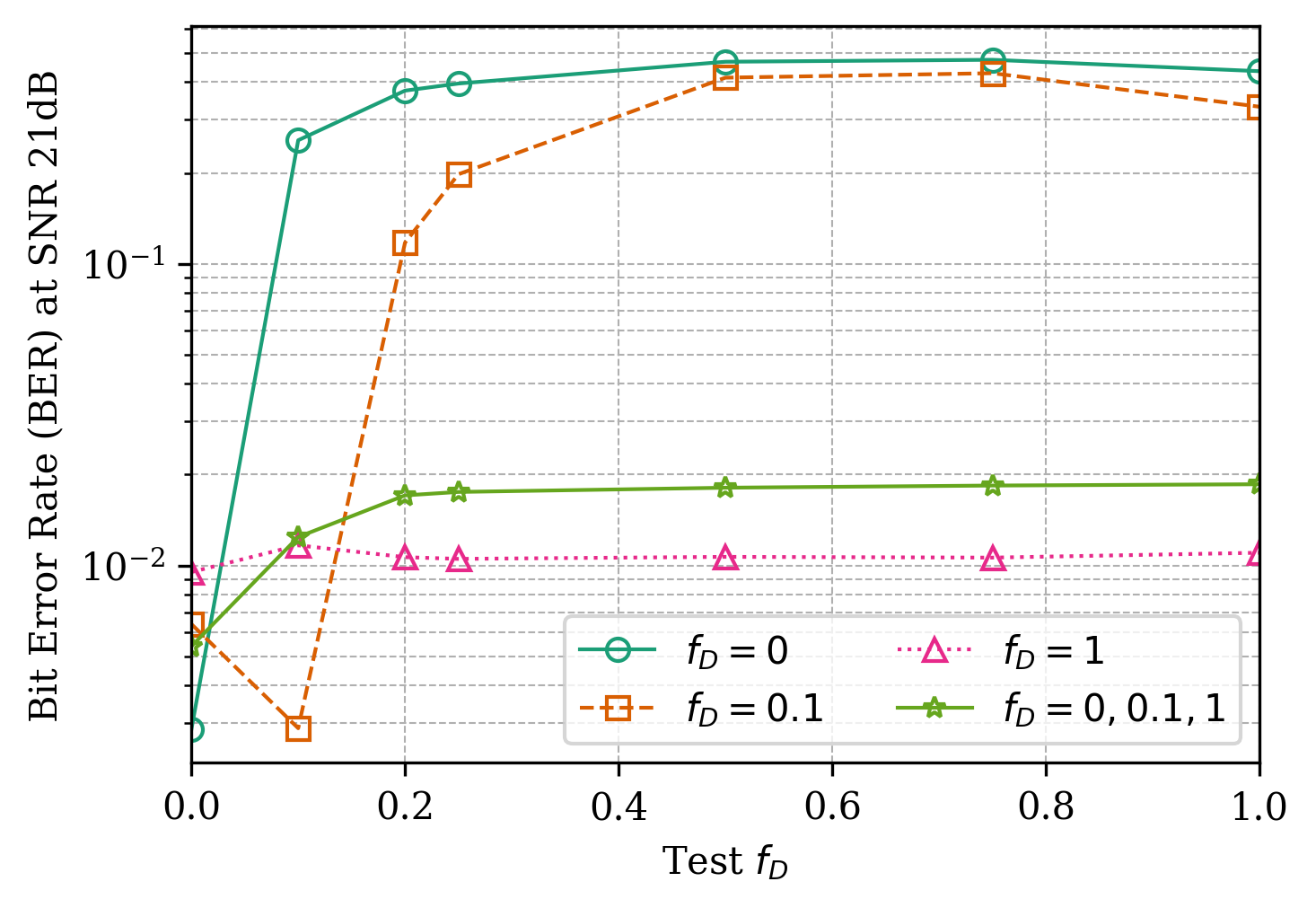}} \hspace{0.1cm}
    \subfloat[\color{black}MCRBs of the codeword obtained by the CNN trained for sensing, compared with baseline sequences. \label{fig:MCRB_K6N9_rho0.0}]{\includegraphics[width=0.3\textwidth]{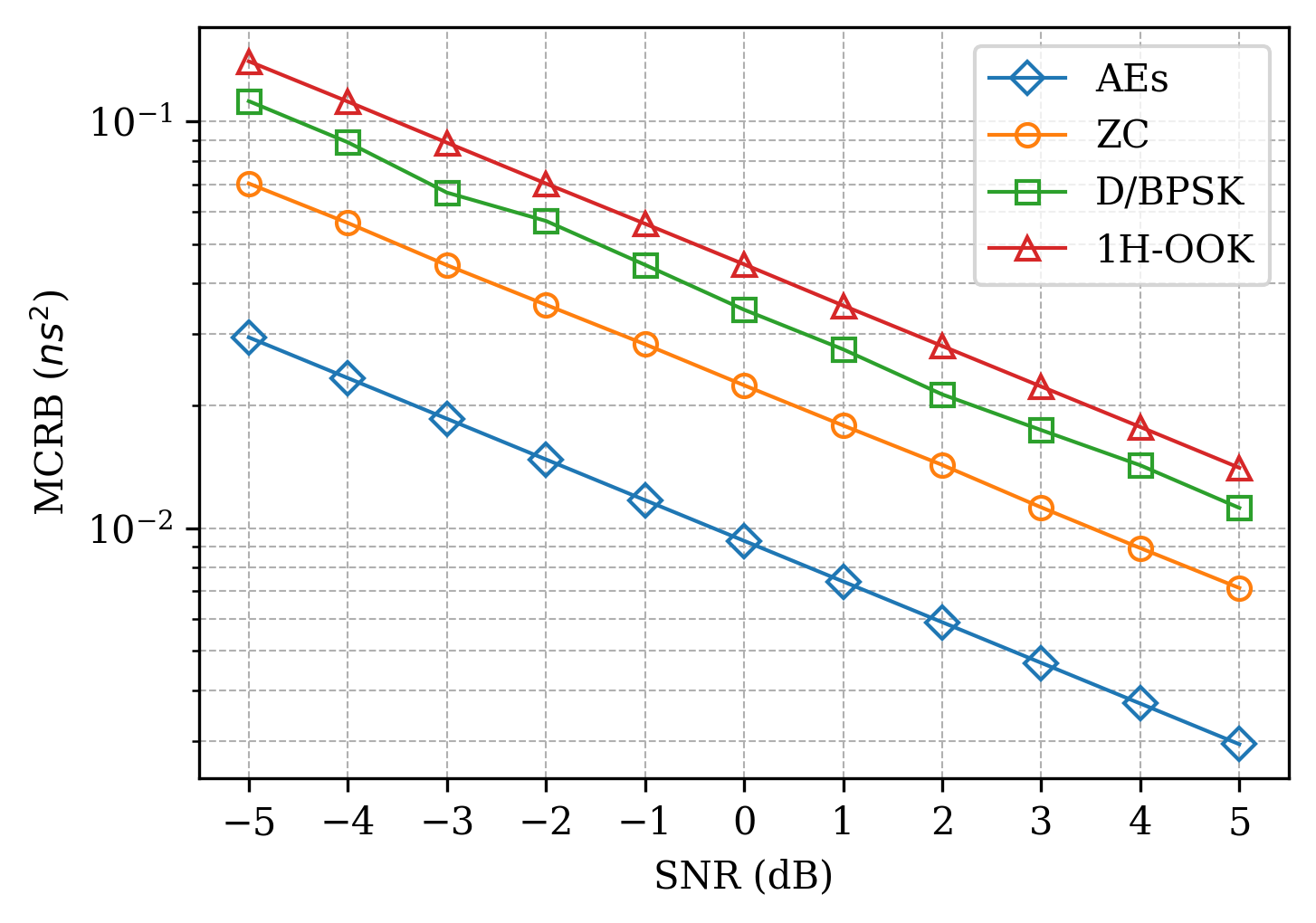}}
\hspace{0.1cm}
    \subfloat[{\color{black}MSE and non-outage MSE of delay estimation using a matched filter-based detector and for a single-target scenario, compared with the MCRB.}\label{fig:MCRB_K6N9_P1_MSE}]{\includegraphics[width=0.32\textwidth]{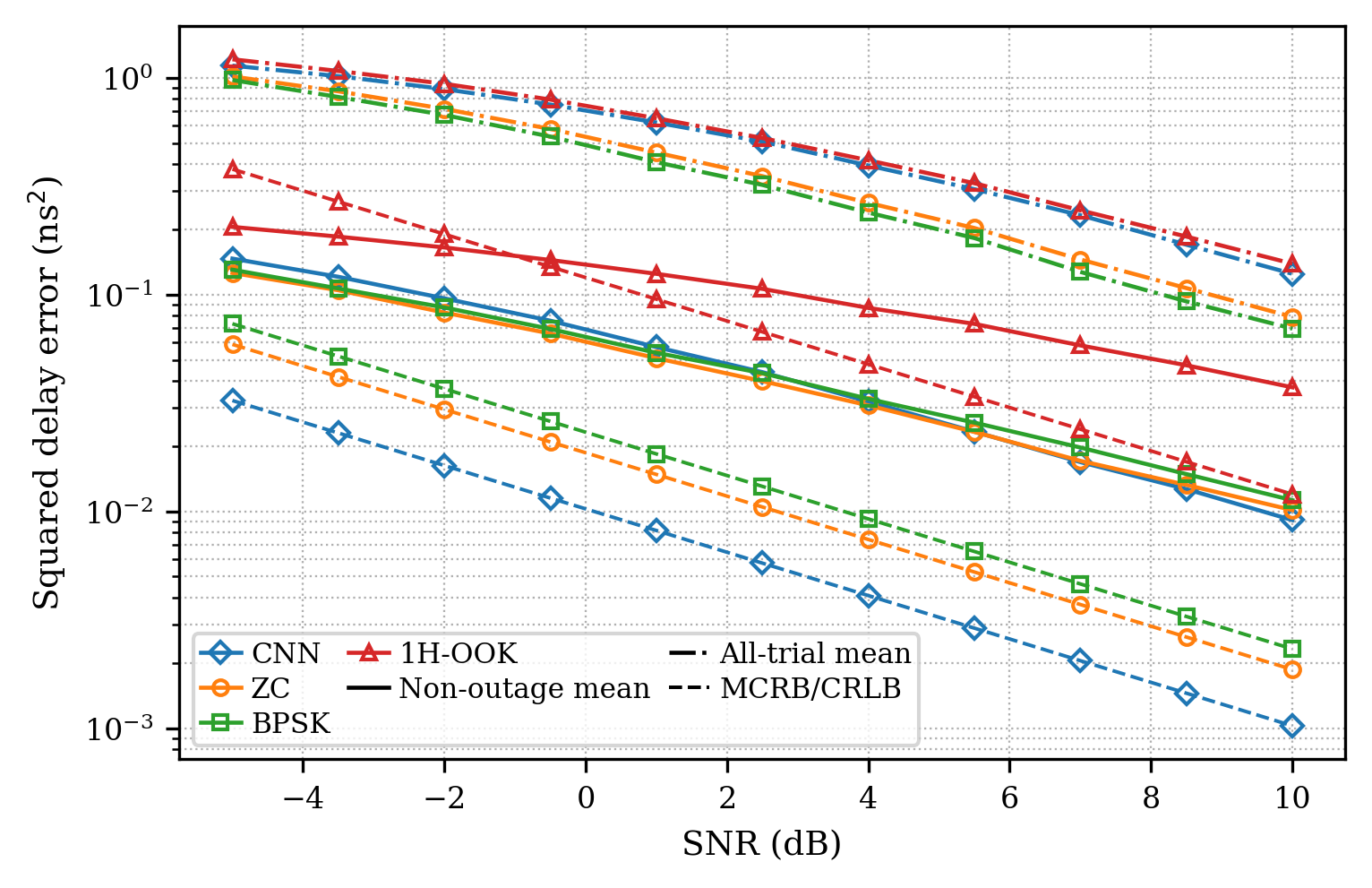}}
\\
\centering
\subfloat[Constellations of sensing-optimal codewords learned by the CNN.\label{fig:LinePlot_K6N9_gamma0}]{\includegraphics[width=2\columnwidth]{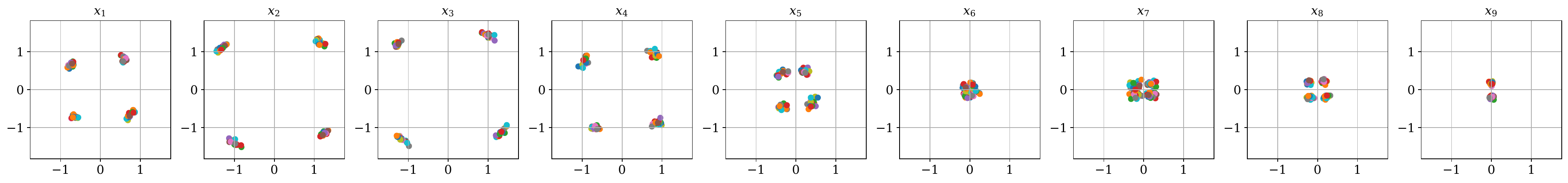}}
\caption{Communication- and sensing-optimal designs are analyzed through BER, MCRB, and constellation visualizations for $K=6$, $N=9$. The MSE result is additionally reported under a simplified single-target scenario ($P=1$) to provide empirical validation.}\label{fig:BER_K6N9_rho1.0}
\end{figure*}

\section{Validation using a Very Short Block Length}\label{sec:sim_toy_example}
This section presents code design results for a representative short blocklength setting. We study communication-, sensing-, and joint-optimal designs to understand the trade-offs between the two functionalities.

For communication, we evaluate NN-based encoders under fully, partially, and uncorrelated fading channels, and compare them with classical coding and modulation schemes. For sensing, we numerically investigate delay estimation performance by using the MCRB and a complementary MSE-based evaluation and analyze the structure of the learned waveforms. Finally, for joint design, we examine how the proposed framework balances communication and sensing objectives through a weighted loss, revealing the achievable trade-off region.

For clarity, we consider a small-scale example with $K=6$ and $N=9$. The NN feature size is set to 16, and the models are trained over SNR uniformly sampled in $[0,30]$~dB. Every model is trained for \{50, 100\} cycles of iterative training with learning rate $\{10^{-4}, 10^{-5}\}$, respectively. {\color{black}In our numerical results we use a Nyquist-root pulse shape  $r(t)=\sinc(t/T_{\rm s})=\frac{\sin{\pi t/T_{\rm s}}}{\pi t /T_{\rm s}}$}

\subsection{Communication Optimal Design}\label{ssec:sim_results_very_short_commopt}
{\color{black}Fig.~\ref{fig:BER_K6N9_gamma1} shows the BER of various AE-based designs and classical channel coding and modulation schemes under three correlation scenarios. The middle figure corresponds to the partially correlated case with $f_{\rm D}=0.1$ in~\eqref{eq_ACF}. The AE is trained separately for each scenario, as the optimal code design depends on the correlation. This dependence is further supported by the robustness test presented in Fig.~\ref{Fig_robustness_test}.

As a classical approach, we use a convolutional code (CC) of rate $2/3$ with generator matrix 
\begin{equation}
G = \begin{bmatrix}
7 & 5 & 0 \\
0 & 6 & 3
\end{bmatrix}_8.
\end{equation}
The CC is well suited for short blocklength transmission due to its low encoding and decoding latency. We implement a terminated CC using the CommPy simulator \footnote{Due to simulator constraints for short lengths, the implementation does not exactly match the considered setting: We use a message length of 24 and a code length of 36. Including 6 termination bits and 3 padding zeros, each codeword spans 45 symbols and is transmitted over 5 blocks of length 9.}.
In contrast, modern coding schemes such as LDPC and polar codes are typically designed for longer blocklengths and often rely on coherent detection or pilot-assisted channel estimation, making them less suitable for the considered short-blocklength noncoherent scenario.

As classical modulation schemes, we consider differential BPSK (DBPSK), BPSK, and binary ASK (BASK). Note that DBPSK requires one reference symbol and  uses a codeword length of $N=10$, while the other schemes use $N=9$.  
The BPSK is detected using a nearest-neighbor rule. In the considered noncoherent setting, the likelihood is invariant to a global phase rotation, leading to an inherent ambiguity between $\mathbf{x}$ and $-\mathbf{x}$ for PSK, which results in an error floor. Meanwhile, BASK is evaluated using the ML detector. 

The three AE architectures exhibit similar BER performance across all scenarios, outperforming the conventional approaches combined with CC in most cases. In the fully correlated case, the AE-based designs achieve nearly identical BER to CC-DBPSK, indicating that they approach the known optimal performance in this regime.  
For partially and uncorrelated fading, the AEs provide significant performance gains over the classical schemes. Among the AEs, the turbo AEs show slightly better performance in the high-SNR regime, while the CNN-based design achieves comparable performance with significantly lower complexity, as it avoids recurrent and iterative structures. The non-zero BER floor at high SNR is a well-known characteristic of noncoherent communication systems, caused by the ambiguity introduced by unknown channel coefficients and random phase rotations~\cite{proakis2008digital,simon2005digital}.

The DBPSK performs well under fully correlated fading at high SNR, but its performance degrades as the channel correlation decreases, as shown in the middle figure, because of error propagation in differential decoding.  
The performance of BPSK and BASK shows less sensitivity to channel correlation, as they operate symbol-wise. The BASK performs relatively better in the uncorrelated high-SNR regime. 

In summary, theA AEs are particularly advantageous in low- and moderate-SNR regimes, especially when channel correlation is weak, where analytical solutions are difficult to obtain.

\begin{figure*}[t!]
\centering
\subfloat[\color{black}BER of the CNN for multiple joint training weight $\gamma$ evaluated for three channel correlation cases.]{
\includegraphics[width=0.32\textwidth]{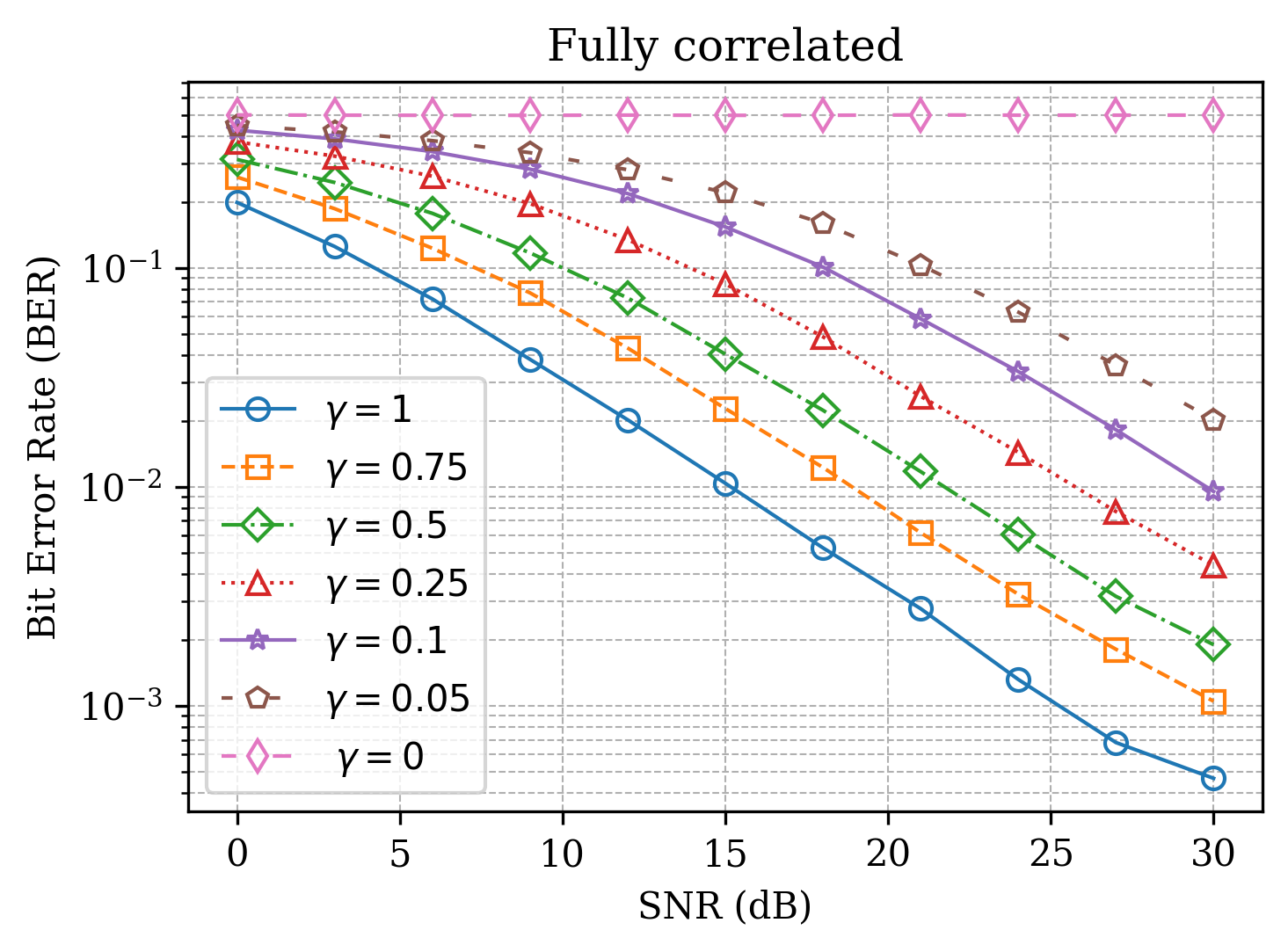} 
\includegraphics[width=0.32\textwidth]{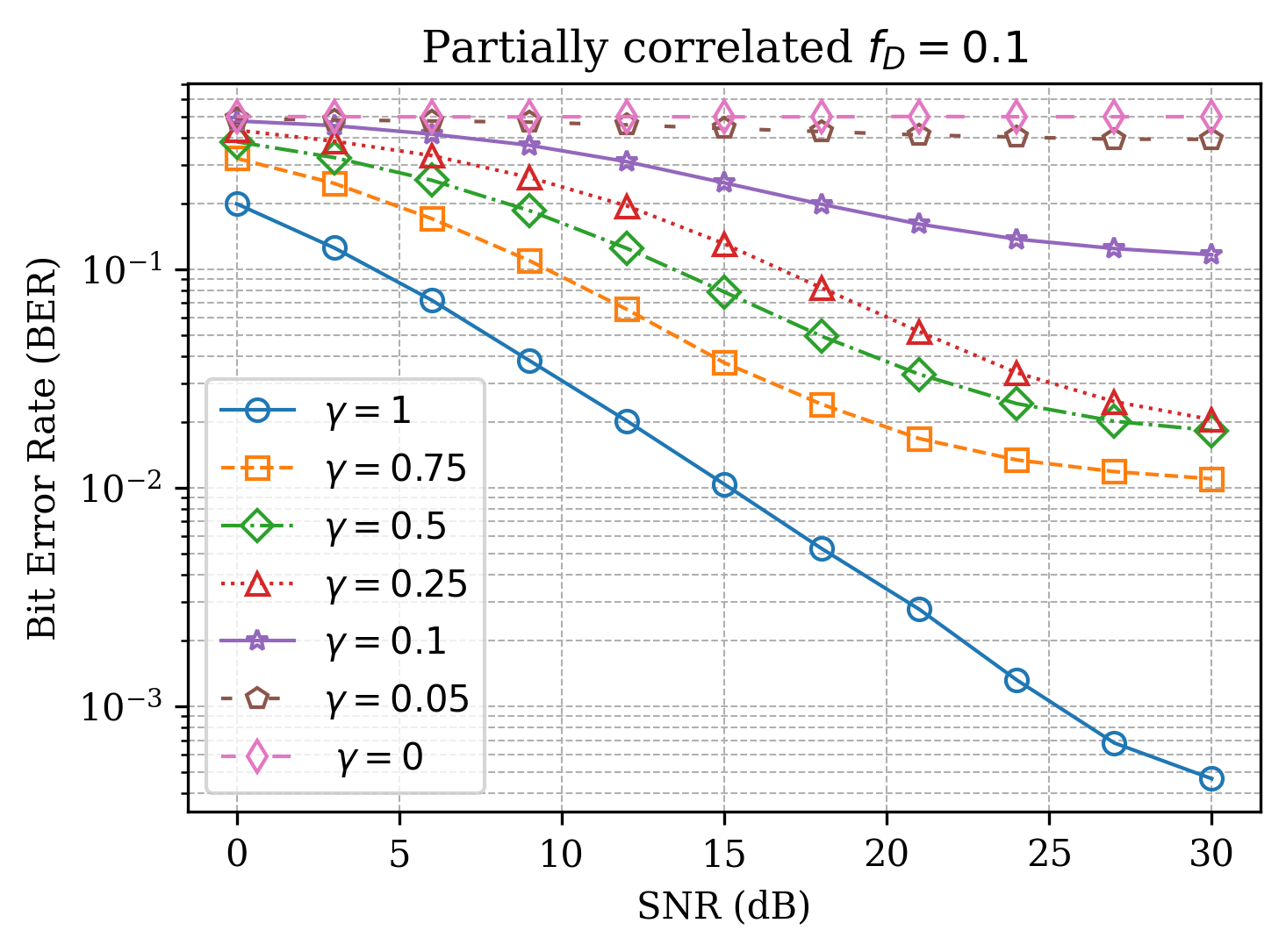}  
\includegraphics[width=0.32\textwidth]{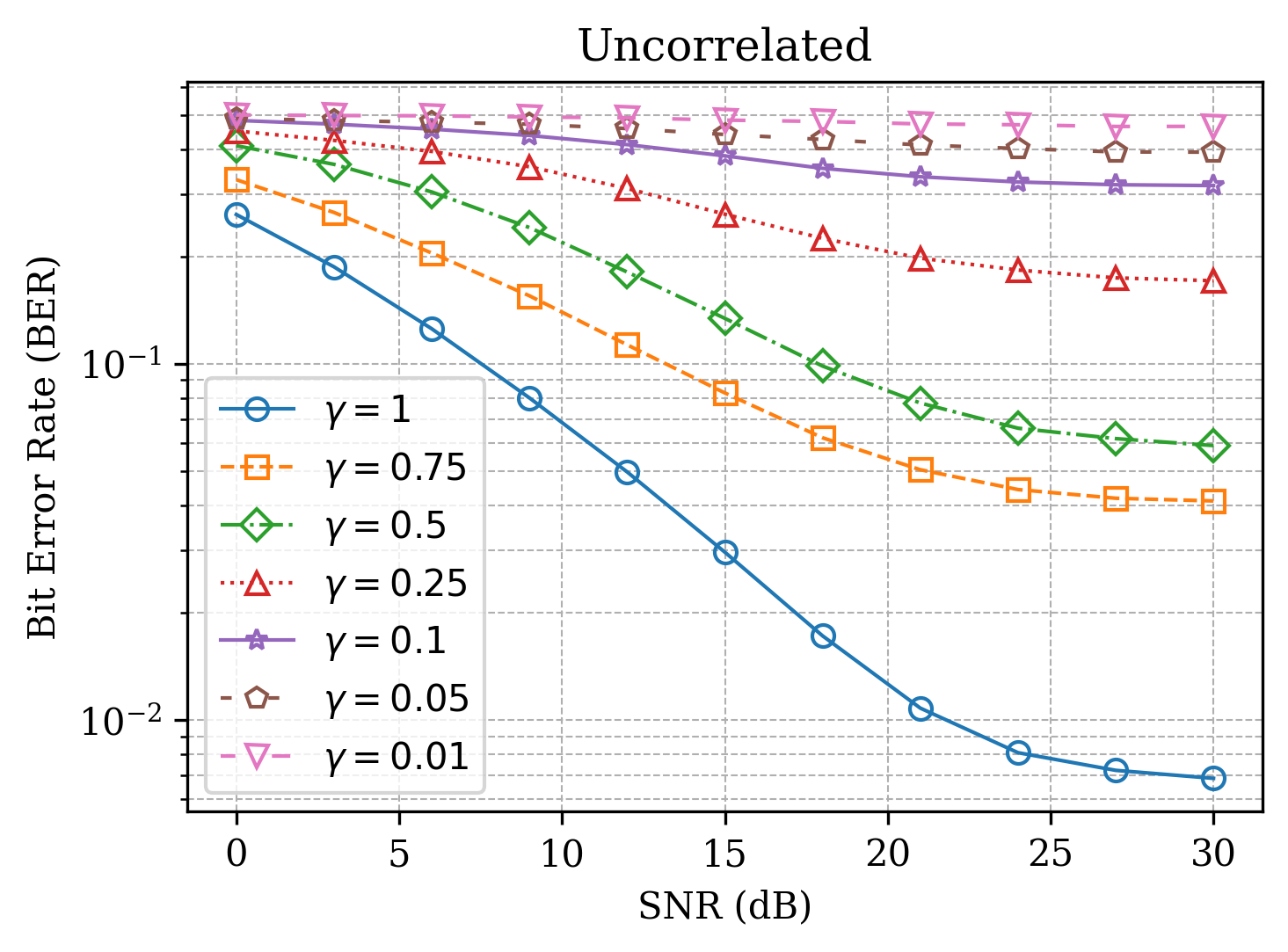} } \\
\centering
\subfloat[{\color{black}MCRB values of the CNN at an SNR of $0$~dB.Due to the linear dependence of the MCRB on SNR, a single representative value is sufficient to characterize the trend.}]{
{\footnotesize
\begin{tabular}{|c|c|c|c|c|c|c|c|c|}
    \hline
    \multicolumn{2}{|c|}{$\gamma$} & 0 & 0.05 & 0.1 & 0.25 & 0.5 & 0.75 & 1\\
    \hline
    \multirow{3}{*}{$\frac{1}{T^2_{\rm s}}$ MCRB (dB)} & $f_{\rm D}=0$ & \multirow{3}{*}{\textbf{\color{black}-20.32}}  & {\color{black}-20.27} & {\color{black}-20.27} & {\color{black}-20.18} & {\color{black}-20.04} & {\color{black}-19.79} & {\color{black}-14.91}\\
    \cline{2-2} \cline{4-9}
     &$f_{\rm D}=0.1$ &  &  {\color{black}-20.32} & {\color{black}-20.27} & {\color{black}-20.22} & {\color{black}-19.96} & {\color{black}-19.47} & {\color{black}-14.61}\\
    \cline{2-2} \cline{4-9}
     &$f_{\rm D}=1$ &  & {\color{black}-20.32} & {\color{black}-20.32} & {\color{black}-20.22} & {\color{black}-19.96} & \color{black}{-19.21} & {\color{red}-10.70} \\
    \hline
\end{tabular}}}
    
\caption{BER and MCRB of the CNN for $N=9$ are measured with varying $\gamma$.
}\label{fig:joint_design_N9}
\end{figure*}

Fig.~\ref{Fig_robustness_test} evaluates the robustness of the learned codebooks under mismatched channel correlation, i.e., when the normalized Doppler value $f_{\rm D}$ used for training differs from that for testing. 
We train three separate CNN models for $f_{\rm D} \in \{0, 0.1, 1\}$, respectively, and an additional CNN model using a mixed dataset that randomly samples these three channel conditions during training. The resulting codebooks are evaluated over a range of $f_{\rm D}$ values at an SNR of $21$~dB.

The results show that the BER performance is in general more robust when the CNN is trained with higher Doppler scenario, and the lowest BER is achieved when $f_{\rm D}$ is matched for training and testing. The model trained on the mixed dataset improves the worst-case performance but yields consistently suboptimal BER across all tested values. These observations indicate that the optimal code design is strongly dependent on the channel correlation characterized by $f_{\rm D}$.}

\begin{figure*}[h!]
\centering
\subfloat[BER of the CNN trained with various joint loss weight $\gamma$, compared to the best baseline of each case.]{
\includegraphics[width=0.32\textwidth]{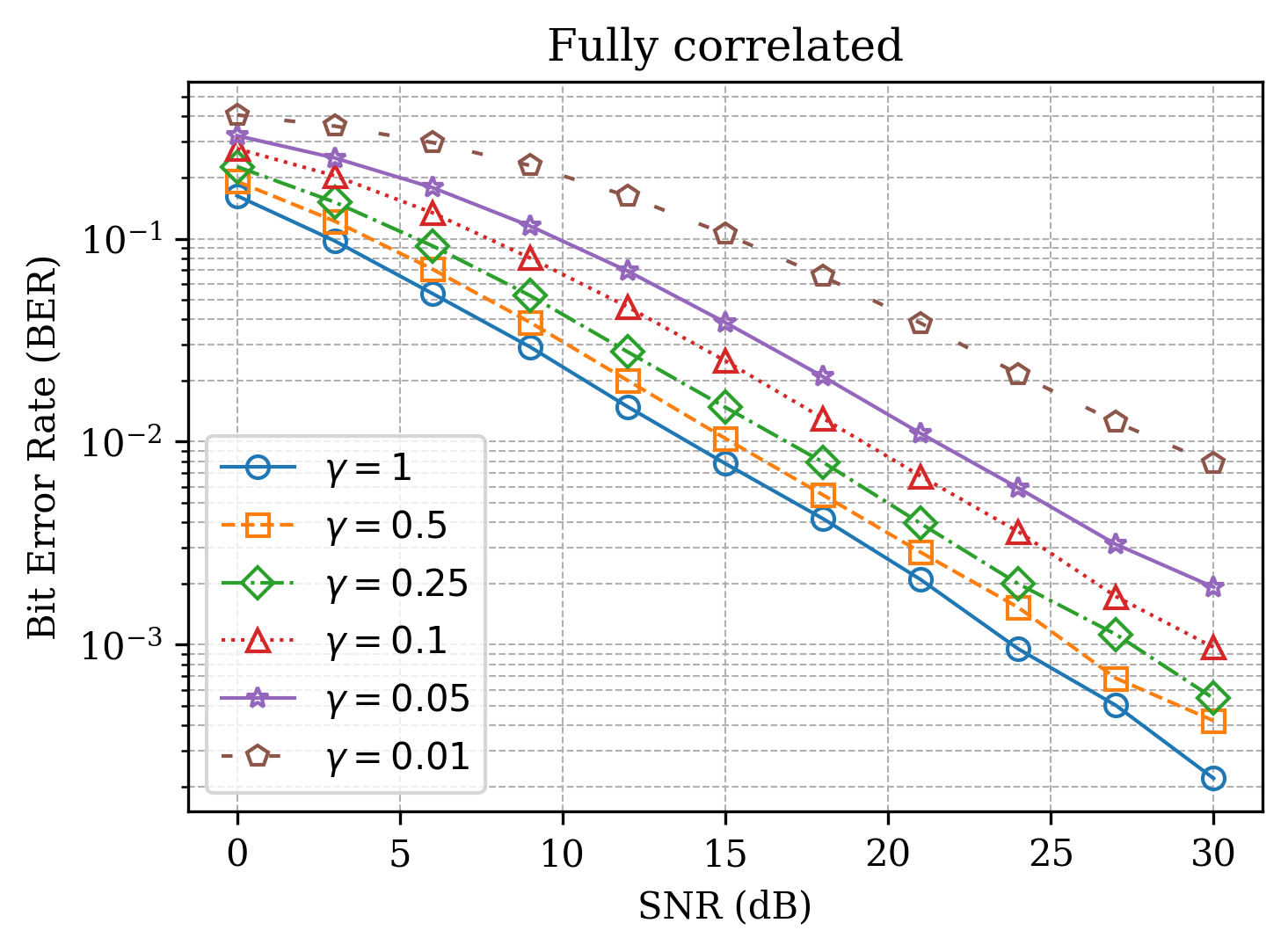}
\includegraphics[width=0.32\textwidth]{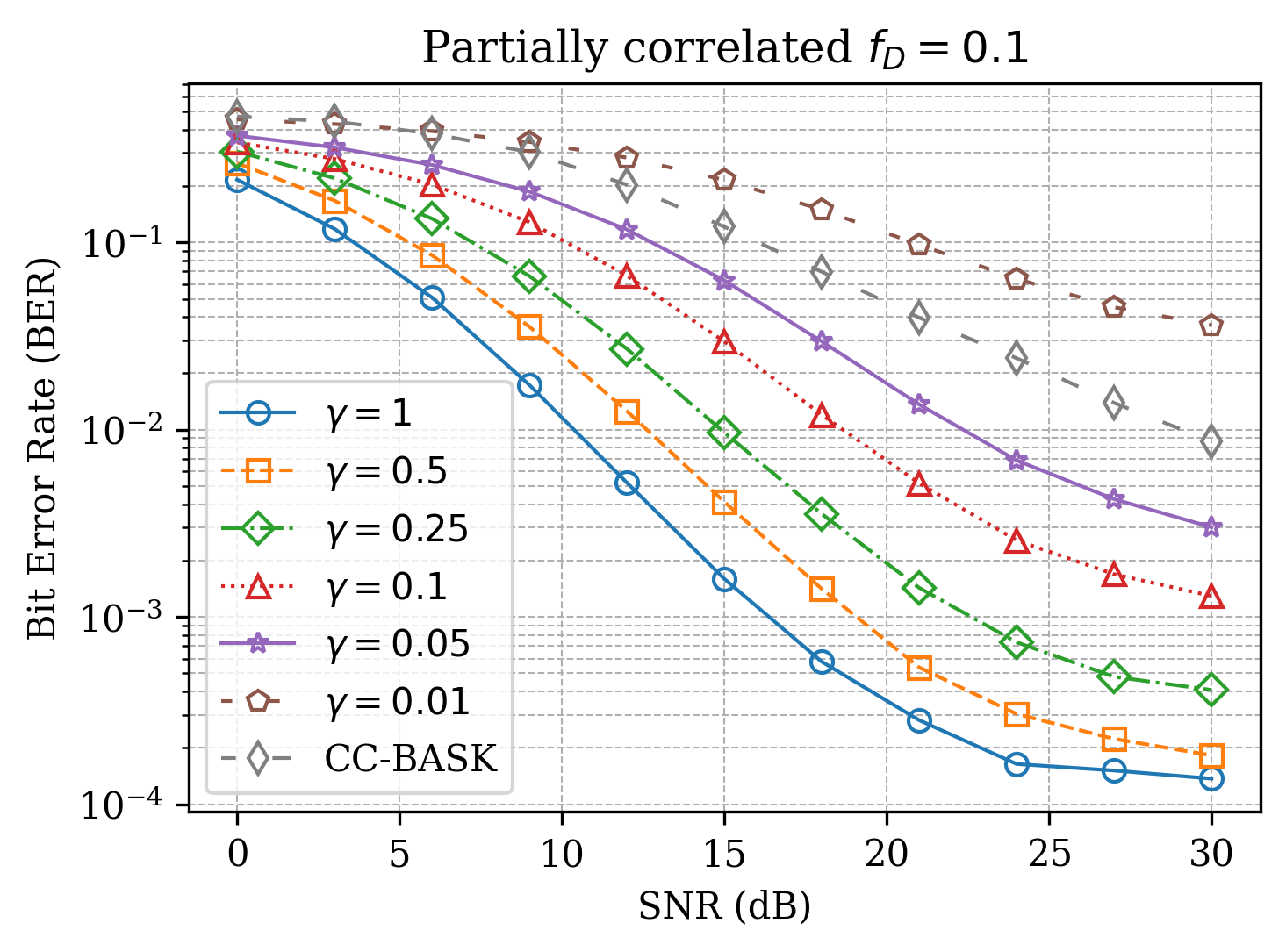}
\includegraphics[width=0.32\textwidth]{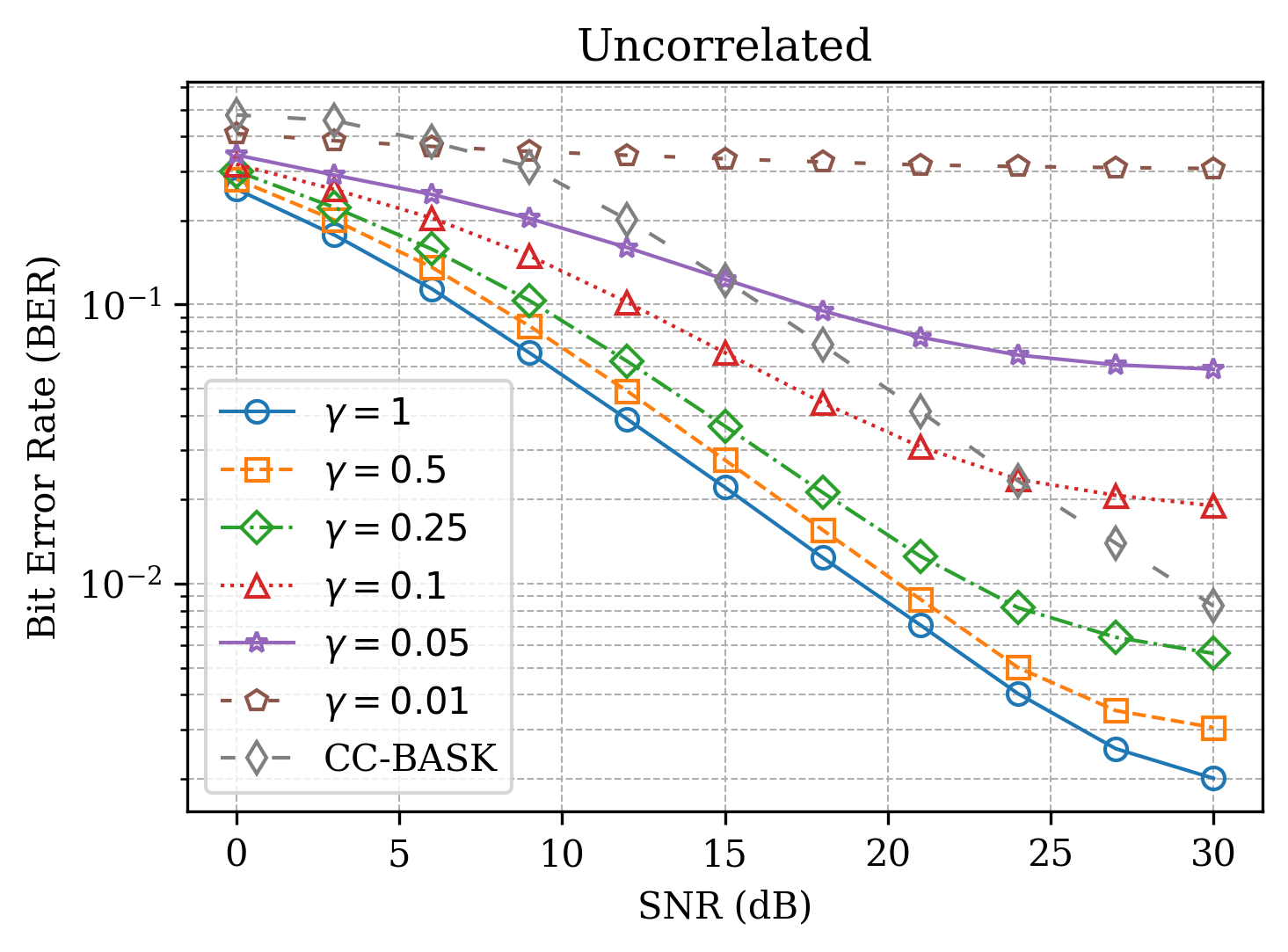}
}

\subfloat[MCRB values of the CNN at an SNR of $0$~dB. Values are reported in dB scale with an offset of $-30\,$dB applied for readability. {\color{black}Due to the linear dependence of the MCRB on SNR, a single representative value is sufficient to characterize the trend.}]{\footnotesize 
\begin{tabular}{|c|c|c|c|c|c|c|c|c|}
    \hline
    \multicolumn{2}{|c|}{$\gamma$} & 0 & 0.01 & 0.05 & 0.1 & 0.25 & 0.5 & 1 \\
    \hline
    
    \multirow{2}{*}{$\frac{1}{T^2_{\rm s}}$MCRB} & $f_{\rm D}=0$ 
    & \multirow{3}{*}{\textbf{\color{black}-1.67}} & {\color{black}-1.55} & {\color{black}-1.43} & {\color{black}-1.31} & {\color{black}-1.02} & {\color{black}-0.66} & 2.88 \\ 
    
    \cline{2-2} \cline{4-9} 
     &$f_{\rm D}=0.1$ 
     &  & {\color{black}-1.55} & {\color{black}-1.37} & {\color{black}-1.19} & {\color{black}-0.76} & {\color{black}-0.09} & {\color{black}2.70} \\

    \cline{2-2} \cline{4-9} 
     (dB, rel. –30 dB) &$f_{\rm D}=1$ 
     &  & {\color{black}-1.61} & {\color{black}-1.25} & {\color{black}-1.02} & {\color{black}-0.60} & {\color{black}-0.13} & {\color{red}2.99} \\
    \hline
\end{tabular}
}
\caption{BERs and MCRBs of the CNN model for $N=90$, measured for varying $\gamma$ values for joint design. 
}\label{fig:joint_design}
\end{figure*}
\subsection{Sensing Optimal Design}
Fig.~\ref{fig:MCRB_K6N9_rho0.0} illustrates the MCRB of the codewords optimized by the AEs and D/BPSK (both DBPSK and BPSK). We included a Zadoff-Chu (ZC) sequence of length $N$, i.e., $x_u[n]=\exp\left(-j\frac{\pi u n(n+c_{\rm f})}{N}\right)$, where $c_{\rm f}=N \mod 2$ and $u=1$. The one-hot encoded on-off keying (1H-OOK) is also tested, where each sequence has $(N-1)$ zeros and one non-zero signal with power $\sqrt{N}$. The ZC and 1H-OOK use 1 and $N$ codewords, respectively, which offer much less communication rates when they are used for communications. 
The BASK is excluded because the MCRB value can explode due to symbol-wise power normalization.
For simulations, we consider $P=5$ paths, and their delay ranges are set by $L_{\rm s}=\{1, 2, 3, 5, 7\}$ and $T_{\rm s}=1$~ns. 
The MCRB values are averaged over the delay distributions of all paths and for random input bit input sequences. 

The MCRBs of the AEs are all identical and the lowest, followed by the ZC, D/BPSK, and 1H-OOK. 
We suppose that the sensing-optimal design is easier to obtain than the communication-optimal design because all AEs achieve the optimal performance without slight differences. 

{\color{black}To empirically validate the MCRB-based performance evaluation, Fig.~\ref{fig:MCRB_K6N9_P1_MSE} reports the delay-estimation MSE obtained with a matched-filter-based detector. The detector evaluates a correlation score between the received signal and a single-path template over hypothesized delays. We consider a single-target setting with $P=1$ and $L_{\rm s}=3$.

Due to block fading, deep fading realizations can lead to outage events with very small effective channel gain, as reflected by the all-trial mean curves. To isolate the non-outage estimation behavior and enable a clearer comparison among sequences, we also report a genie-aided non-outage mean. In this evaluation, a trial is classified as an outage if the absolute delay-estimation error exceeds $T_{\rm s}$; the remaining trials are used to compute the curves labeled as non-outage mean. The MCRB and CRLB curves are computed from \eqref{eq_CRLB_FIM} and \eqref{eq_MCRB}, respectively.

For the single-target case, the MCRB and CRLB curves overlap, which is consistent with the theoretical derivation. The non-outage MSE showed similar performance for the CNN, ZC, and BPSK, whereas the 1H-OOK has noticeably higher values. A non-negligible gap remains between the simulated non-outage MSE and the corresponding bounds for all sequences, which can be attributed to the short block length and the practical resolution limits of the detector.

We illustrate the learned codewords in Fig.~\ref{fig:LinePlot_K6N9_gamma0}, where each subfigure shows the I/Q representation of a codeword consisting of $N=9$ symbols, and different colors correspond to different codewords.

The codewords exhibit PSK-like structures with varying amplitudes across symbols. The phase ambiguity arises from the rotational invariance discussed in Sect.~\ref{ssec:heuristic_understanding_sensing}. 
The non-uniform amplitude across symbols reflects the sensing-driven optimization, where power is allocated unevenly over time to improve estimation performance, rather than enforcing constant-modulus (CM) constraints. 

We note that CM waveforms are often preferred in practice due to their favorable PAPR and power amplifier (PA) efficiency. The unconstrained design adopted here may therefore incur practical costs due to nonlinear distortion. However, this work focuses on the signal design aspect and allows the AE to learn the waveform freely. Incorporating hardware-aware constraints is an interesting direction for future work.
}

\begin{figure}
\centering\includegraphics[width=0.35\textwidth,valign=t]{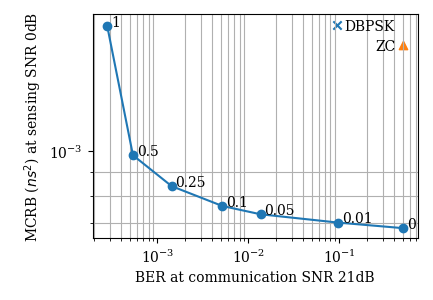}
\caption{\color{black}Pareto trade-off between BER and MCRB for different values of the joint loss parameter $\gamma$, evaluated with the CNN model for $K=60, N=90$ under a partially correlated channel with $f_{\rm D}=0.1$. Each circular marker on the curve corresponds to a specific $\gamma$ value, whereas a cross and a triangle marker show DBPSK and ZC, respectively.}\label{fig:pareto_curve_f01}\end{figure}

\subsection{Joint Design}\label{Sect_short_joint_design}
We test the joint design performance of the AE by using only the CNN model because it performs well with relatively lower complexity than the other AEs. 
For the test, we use the same communication and sensing channel as above.
Fig.~\ref{fig:joint_design_N9} shows the BER of the AE for varying $\gamma$ values at SNR between 0 to 30~dB and the MCRBs at 0~dB SNR\footnote{Remark that the noise power for the communication and sensing channels are distinct, and the SNR can be accordingly different.}. Since MCRB values are very small, we visualize $\frac{1}{T_{\rm s}}$MCRB for readability. To compensate the small scale of the sensing loss, we introduce additional scaling to the sensing loss of 1000. 
{\color{black}The MCRB decreases linearly with SNR, with a slope determined by $1/N_{0,\rm s}$. Therefore, a single representative value is sufficient to characterize its behavior across SNR, and we report the results at a reference SNR of $0$~dB.}
For comparison, the MCRBs of ZC, BPSK, and 1H-OOK are observed to be $\{-16.48, -14.55, -13.52\}$~dB, respectively. The table is color-coded according to these reference values: blue indicates performance better than ZC, black indicates values between ZC and BPSK, and red indicates values worse than BPSK.
 
{\color{black}In all three cases of $f_{\rm D}$, a clear trade-off between BER and MCRB is observed as the weighting parameter $\gamma$ varies, which becomes more pronounced under less correlated fading.

\section{Simulation Results for a Moderately Short Block Length}\label{sec:sim_longer_blocklength}
In this section, we focus on the case $N=90$ and $K=60$ to design a moderately short code for practical scenarios. We present the performance of the CNN model, since we observe that the CNN architecture achieves comparable performance while having lower implementation complexity.
For the simulations in this section, we use the feature size of $60$ in all AE models and the same setting for SNR, epochs, learning rates, and the scaling of MCRB in the joint loss as in Sect.~\ref{sec:sim_toy_example}.

Fig.~\ref{fig:joint_design} shows the BERs and MCRBs of the CNN model for various $\gamma$ values. For communication channels, we use three scenarios of fading correlation, $f_{\rm D}=0, 0.1, 1$. For sensing channel, we trained and tested with four paths, $P=4$, with $L_{\rm s}=5,10,15,20$. For the ZC sequence, BPSK, and the 1H-OOK, the MCRBs are measured as $\{2.30, 2.74, 5.99\}$~dB relative to –30 dB, respectively. The MCRB values in the table are color-coded following the same rule as used in Sect.~\ref{Sect_short_joint_design}.

The tendency looks similar to that in Fig.~\ref{fig:joint_design_N9} where BERs become gradually larger as $\gamma$ decreases, and the trade-off is stronger for less correlation in fading. The MCRBs of the CNN also grow as $\gamma$ increases, but they are lower than that of ZC sequence for all $\gamma<1$, which are colored blue. 

The results show that the CNN with joint loss effectively balances sensing and communication through the control parameter $\gamma$, achieving operating points along the trade-off curve. These observations are consistent with prior work~\cite{kim2024short}. Compared to the shorter blocklength case ($N=9$ in Fig.~\ref{fig:joint_design_N9}), the trade-off is milder, and joint designs remain close to both sensing- and communication-optimal performance. For instance, $\gamma=0.5$ achieves BER close to the optimal curve while maintaining a lower MCRB than the ZC sequence.

This trade-off is illustrated by the Pareto curve in Fig.~\ref{fig:pareto_curve_f01}, where the operating points form a monotonically decreasing trend. The results indicate a continuous transition between sensing- and communication-oriented designs as $\gamma$ varies. In contrast, DBPSK and ZC, corresponding to single points marked by an `x' and a triangle, lie outside the Pareto curve, reflecting less efficient trade-offs under the considered setting.}

\section{Conclusion}
This paper proposed an autoencoder-based framework for joint waveform design in ISAC, targeting short blocklength and noncoherent communication scenarios. An MCRB for multi-target delay estimation was derived and incorporated into the training objective, enabling joint optimization of communication and sensing performance.

The results show that the learned waveforms outperform conventional designs and provide a controllable trade-off between BER and sensing accuracy through a single parameter. The proposed approach is particularly effective in regimes where classical design is challenging, such as short blocklength and partially or uncorrelated fading.

Overall, the framework demonstrates that data-driven signal design can effectively address the coupled communication–sensing optimization problem and offers a flexible tool for ISAC waveform design under practical constraints.
\appendix

\subsection{Proof of Theorem~\ref{thm_MCRB}}\label{proof_thm_MCRB}
We obtain the corresponding FIM and derive the MCRB from it. To efficiently represent the Fischer information matrix as in~\cite{rogalin2014scalable}, we first solve the integral of each element of the channel matrix ${\bf H}_{\rm s}$ in \eqref{rad_g_m_n} {\color{black}for an arbitrary Nyquist root pulse $r(t)$ in the following lemma. 

\begin{lemma}\label{lemma:g_m-n}
Let the $(m,n)$-th element of channel matrix ${\bf H}_{\rm s}$ be defined as in~\eqref{rad_g_m_n}.
For an energy-normalized Nyquist root pulse such that $R(f) = \sqrt{P(f)}$ for a Nyquist pulse $P(f)$ and $\int_{\infty}^{\infty}|r(t)|^2dt=1$, the channel matrix ${\bf H}_{\rm s}$ can be obtained by its elements satisfying
\begin{align}
g_{m-n} = \sum_{i=1}^{P} h_{\rm s}^{(i)} p\left((m-n)T_{\rm s}-\tau^{(i)}\right).\label{eq_g_m-n}
\end{align}
\end{lemma}
\proof 

For arbitrary $\Delta t$,  we have
\begin{align}
\int_{-\infty}^{\infty} r(t) r^*(t + \Delta t) dt = \int_{-\infty}^{\infty} |R(f)|^2 e^{-j2\pi f \Delta t} df. 
\end{align}
With $t' \coloneqq t-nT_{\rm s}-\tau^{(i)}$, the integral can be solved as
\allowdisplaybreaks
\begin{align}
& \int_{-\infty}^{\infty} r(t - nT_{\rm s} - \tau^{(i)}) r^*(t - mT_{\rm s}) dt\\
= & \int_{-\infty}^{\infty} r(t') r^*(t' +nT_{\rm s} - mT_{\rm s} + \tau^{(i)}) dt\\
= & \int_{-\infty}^{\infty} \{R(f)\}^2 e^{ -j2\pi f ( (m-n)T_{\rm s} - \tau^{(i)}) } df\\
= & \mathcal{F}^{-1}\{P(f)\}\left((m-n)T_{\rm s} - \tau^{(i)}\right)\\
= & p\left((m-n)T_{\rm s} - \tau^{(i)}\right).
\end{align}
 Substituting the integral in \eqref{rad_g_m_n} completes the proof.
\qedhere

The integral in each term becomes a Nyquist pulse function that has the peak, when $(m-n)T_{\rm s}$ is equal to the delay $\tau^{(i)}$. The contribution of each path is delayed by $\tau^{(i)}$ and spread out in the shape of the Nyquist pulse.} 
In the next lemma, we introduce a decomposed form of ${\bf H}_{\rm s}{\bf x}$ as used in~\cite{rogalin2014scalable}.

\begin{lemma}\label{lemma:newform}
For ${\bf H}_{\rm s}$ 
defined in~\eqref{eq_g_m-n}, ${\bf H}_{\rm s}{\bf x}={\bf X}'\bm\Gamma(\bm\tau) \bf{h}_{\rm s}$ holds, where $\bf{X}'$ is an $N\times (2N-1)$ matrix satisfying
\begin{align}
{\bf{X}'} &=
\begin{bmatrix}
    x_{N} & x_{N-1} & x_{N-2} & \cdots      & \phantom{0} & 0       & 0\      \\
    0     & x_{N}   & x_{N-1} & \cdots      & \phantom{0} &  0      & 0\      \\
    0     & 0       & x_{N}   & \cdots      & \phantom{0} &  0      & 0\      \\
    \vdots&  \vdots &\cdots   & \phantom{0} & x_2         &  x_{1}  & 0\      \\
    0     & 0       &\cdots   & \phantom{0} & x_3         &  x_{2}  &  x_{1}\ \\
\end{bmatrix},\\
{\bf{h}}_{\rm s} &= [h_{\rm s}^{(1)}, h_{\rm s}^{(2)}, \cdots,  h_{\rm s}^{(P)}]^{\rm T},
\end{align}
${\bf{x}}=(x_1, x_2, \dots, x_N)^{\rm H
}\in\mathbb{C}^{N\times 1}$, and $\bm\Gamma(\bm \tau)$ is a $(2N-1)\times P$ matrix, whose $(i,j)$-th element is ${\color{black}p\Big((i-N)T_{\rm s}-\tau^{(j)}} \Big)$. 
\end{lemma}
The representation separates the contribution of delay $\bm\tau$ and fading ${\bf h}_{\rm s}$ and is convenient for taking partial derivatives. The following theorem provides the FIM that is defined in~\eqref{eq_MCRB_} based on the formula above.
\begin{theorem}\label{lemma_Fisher_element}
    {\color{black}Let $\bm{\tau} = \Big( \tau^{(1)}, \tau^{(2)}, \cdots, \tau^{(P)}\Big)$ be the parameter vector of interest, and the random fading ${\bf h}_{\rm s}$ and codeword ${\bf x}$ are considered as nuisance parameters as represented in~\eqref{eq_MCRB_}.
    For efficient representation, we define $\bm \theta := (\bm\tau, {\bf h}_{\rm s}, {\bf x})$ and ${\bf m}(\bm\theta):={\bf X}'\bm\Gamma(\bm\tau){\bf h}_{\rm s}$.
    The $(a,b)$-th element of the corresponding FIM defined in~\eqref{eq_MCRB_element}} satisfies
    and the partial derivatives are as follows. 
\begin{align}
    [{\bf J}(\bm\theta)]_{a,b} &=\frac{2}{N_{0, \rm s}}\Re{\left(\left(\frac{\partial {\bf{m}}(\bm \theta)}{\partial\tau^{(a)}}\right)^{\!\!\rm H} \frac{\partial {\bf{m}}(\bm \theta)}{\partial \tau^{(b)}}\right)}\\
    \frac{\partial{\bf m}(\bm\theta)}{\partial\tau^{(a)}} &=  h^{(a)}_{\rm s}{\bf X}'\frac{\partial{\bm\Gamma}(\bm\tau)}{\partial\tau^{(a)}} \label{eq_partial_tau}
\end{align}
\end{theorem}
\proof {\color{black}Given $\bm{\tau}$, $\bf{h}_{\rm s}$, and ${\bf x}$,
the received signal} $\bf{r}_{\rm s}$ in \eqref{eq:io_relation} follows a complex Gaussian distribution $CN({\bf{m}}(\bm \theta), N_{0, \rm s}\bf{I})$. 
In~\cite{kay1993fundamentals}, it is proven that when the likelihood is a multivariate Gaussian PDF with its mean depending on parameter vector $\bm \tau$, then the $(a,b)$-th entry of the FIM $\bf{J}( \bm \theta)$ is 
\begin{align}
    [{\bf J}(\bm \theta)]_{a,b} &=2\Re{\Big(\Big(\frac{\partial \bf{m}(\bm \theta)}{\partial\tau^{(a)}}\Big)^{\!\rm H}(N_{0, \rm s}{\bm{I}})^{-1} \frac{\partial {\bf{m}}(\bm \theta)}{\partial \tau^{(b)}}\Big)}\\
    &=\frac{2}{N_{0, \rm s}} \Re{\Big( \Big(\frac{\partial \bf{m}(\bm \theta)}{\partial\tau^{(a)}}\Big)^{\!\!\rm H} \frac{\partial \bf{m}(\bm \theta)}{\partial \tau^{(b)}} \Big)}. \label{eq:diagonal_element}
\end{align}
Each partial derivative can be obtained for $a=1,2,\dots,P$ as
\begin{align}
&\frac{\partial \bf{m}(\bm\theta)}{\partial\tau^{(a)}}=\frac{\partial\ \ \ }{\partial\tau^{(a)}} ({\bf X}'\bm\Gamma(\bm\tau){\bf h}_{\rm s} )\\
=&
\begin{bmatrix}
    \sum_{n=1}^{N} x_{n} \! \sum_{i=1}^{P} \!h_{\rm s}^{(i)}  \frac{\partial\ \ \ }{\partial\tau^{\bm (a)}} {\color{black}p\left((1-n)T_{\rm s}-\tau^{(i)} \right)} \\
    \sum_{n=1}^{N} x_{n}\! \sum_{i=1}^{P} \!h_{\rm s}^{(i)}  \frac{\partial\ \ \ }{\partial\tau^{\bm (a)}} {\color{black}p\left((2-n)T_{\rm s}-\tau^{(i)} \right)} \\
    \vdots \\
    \sum_{n=1}^{N} x_{n} \! \sum_{i=1}^{P} \!h_{\rm s}^{(i)}  \frac{\partial\ \ \ }{\partial\tau^{\bm (a)}} {\color{black}p\left((N-n)T_{\rm s}-\tau^{(i)} \right)} 
\end{bmatrix}.
\end{align}
Note that only the $a$-th path survives because the partial derivatives of the other paths are all zero. {\color{black}Hence, the multiplied vector ${\bf h}_{\rm s}$ becomes a scalar coefficient $h_{\rm s}^{(a)}$, which completes the proof of~\eqref{eq_partial_tau}.}
\qedhere

The $(a,b)$-th entry of the FIM can be computed as follows by using \eqref{eq_partial_tau} from the above theorem. 
\begin{align}
&{\color{black}[{\bf J}(\bm \tau,{\bf h}_{\rm s}, {\bf x})]_{a,b}}  \\
=\ &\frac{2}{N_{0, \rm s}}(h_{\rm s}^{(a)})^{*}h_{\rm s}^{(b)}\left(\frac{\partial{\bm\Gamma}(\bm\tau)}{\partial\tau^{(a)}}\right)^{T}{\bf X}'^{\rm H}{\bf X}'\frac{\partial{\bm\Gamma}(\bm\tau)}{\partial\tau^{(b)}}.\label{eq_FIM_element}
\end{align}
The expectation over ${\bf h}_{\rm s}$ in~\eqref{eq_MCRB_} cancels out non-diagonal entries in the FIM because $\mathbb{E}_{{\bf h}_{\rm s}}[(h_{\rm s}^{(a)})^{*}h_{\rm s}^{(b)}]=\mathbbm{1}_{a=b}$. This indicates that the FIM is a diagonal matrix, and taking the inverse is significantly simplified.
The expected value of the diagonal element can be obtained in an explicit form as follows:
\begin{align}
&{\color{black} \mathbb{E}_{{\bf h}_{\rm s}}[{\bf J}(\bm \tau,{\bf h}_{\rm s}, {\bf x})]_{a,a} }=\frac{2}{N_{0, \rm s}}\left\Vert {\bf X}'\frac{\partial{\bm\Gamma}(\bm\tau)}{\partial\tau^{(a)}}\right\Vert^2\\
=&\frac{2}{N_{0, \rm s}} \sum_{m=1}^N \!\left\Vert\sum_{n=1}^N x_n \frac{\partial}{\partial \tau^{(a)}}{\color{black}p\left((m-n)T_{\rm s}-\tau^{(a)} \right)}\right\Vert^2.\\
=&\frac{2}{N_{0, \rm s}}\sum_{m=1}^N \!\left\Vert\sum_{n=1}^N x_n {\color{black}p'\left((1-n)T_{\rm s}-\tau^{(a)} \right)}\right\Vert^2,
\end{align}
{\color{black}where $p'(t) = \frac{\partial p(t)}{\partial t}$.}
The MCRB in~\eqref{eq_MCRB} can be obtained by the inverse of the FIM, which has the diagonal element equal to the multiplicative inverse of the equation above.

\subsection{Likelihood of Correlated Fading Channel}\label{appendix:likelihood_specialcase}
Recall that ${\bf{R}}_{{\bf{r}}_{\rm c}|{\bf{x}}}=N_{0, \rm c}\mathbf{I}_N+{{\bf{X}}\mathbf{R}_{\rm c}{{\bf{X}}^{\rm{H}}}}$. 
For the derivation, we break down $\mathbf{R}_{\rm c}$ into a diagonal matrix and an outer product of two vectors, i.e., $\mathbf{R}_{\rm c}=(1-\rho_{\rm{c}})\mathbf{I}_N+\rho_{\rm{c}}\mathbf{1}_N\mathbf{1}_N^{\rm{T}}$, where $\mathbf{1}_N$ is an $N$-dimensional column vector whose elements are all ones, i.e., $\mathbf{1}_N=(1, 1, \dots, 1)^{\rm T}$, and $\rho_{\rm c}=0,1$ characterizes un- and fully-correlated cases, respectively. 
Then, the conditional covariance ${\bf{R}}_{{\bf{r}}_{\rm c}|{\bf{x}}}$ can be also broken down into two parts as
\begin{align}
{\bf{R}}_{{\bf{r}}_{\rm c}|{\bf{x}}}&= N_{0, \rm c}\mathbf{I}_N +\mathbf{X}\left((1-\rho_{\rm{c}})\mathbf{I}_N+\rho_{\rm{c}}\mathbf{1}_N\mathbf{1}_N^{\rm{T}}\right)\mathbf{X}^{\rm H} \\
& = \left(N_{0, \rm c}\mathbf{I}_N+ (1-\rho_{\rm{c}})\mathbf{X}\mathbf{X}^{\rm H}\right) +\rho_{\rm{c}}\mathbf{x}\mathbf{x}^{\rm H}.\label{eq_det_proof1}
\end{align}
For simpler representation, let $\mathbf{A}$ an $N\times N$ diagonal matrix whose $i$-th diagonal element $A_{i,i}$ satisfies $A_{i,i}=N_{0, \rm c}+(1-\rho_{\rm c})|x_i|^2$, i.e., $\mathbf{A}\coloneqq N_{0, \rm c}\mathbf{I}_N+ (1-\rho_{\rm{c}})\mathbf{X}\mathbf{X}^{\rm H}$. We solve the determinant $\det({\bf{R}}_{{\bf{r}}_{\rm c}|{\bf{x}}})$ and the exponent term ${- {\bf{r}}_{\rm{c}}^{\rm{H}}{\bf{R}}_{{\bf{r}}_{\rm c}|{\bf{x}}}^{ - 1}{{\bf{r}}_{\rm{c}}}}$ in \eqref{ML_rule_prob} based on this. 

The determinant $\det({\bf{R}}_{{\bf{r}}_{\rm c}|{\bf{x}}})$ can be derived with $\mathbf{A}$ by using the matrix determinant lemma\footnote{For an invertible matrix $\mathbf{A}$ and two column vectors $\mathbf{u}$ and $\mathbf{v}$, $\det(\mathbf{A}+\mathbf{u}\mathbf{v}^{\rm H})=\det(\mathbf{A})(1+\mathbf{v}^{\rm H}\mathbf{A}^{-1}\mathbf{u})$.  } as follows. 
\begin{align}
&\det({\bf{R}}_{{\bf{r}}_{\rm c}|{\bf{x}}})\\
\hspace{-1cm}=&\Big(\!1\!+\!\!\sum_{i=1}^N 
 \frac{\rho_{\rm{c}}|x_i|^2}{N_{0, \rm c} \!+\!(1\!-\!\rho_{\rm c})|x_i|^2} \!\Big) \!\prod_{i=1}^N \!\left( N_{0, \rm c}\! +\!(1\!-\!\rho_{\rm c})|x_i|^2\right).\label{eq_det}
\end{align}
The exponent term can be solved by computing the inverse of ${\bf{R}}_{{\bf{r}}_{\rm c}|{\bf{x}}}$ by the Sherman–Morrison formula\footnote{For an invertible matrix $\mathbf{A}$ and two column vectors $\mathbf{u}$ and $\mathbf{v}$, $(\mathbf{A}+\mathbf{u}\mathbf{v}^{\rm H})^{-1}=\mathbf{A}^{-1}-\frac{\mathbf{A}^{-1}\mathbf{u}\mathbf{v}^{\rm H}\mathbf{A}^{-1}}{1 +\mathbf{v}^{\rm H}\mathbf{A}^{-1}\mathbf{u}}$} as
\begin{align}
{\bf{R}}_{{\bf{r}}_{\rm c}|{\bf{x}}}^{-1} &= (\mathbf{A} + \rho_{\rm{c}}\mathbf{x}\mathbf{x}^{\rm H} )^{-1} =\mathbf{A}^{-1} - \frac{\mathbf{A}^{-1}\rho_{\rm{c}}\mathbf{x}\mathbf{x}^{\rm H}\mathbf{A}^{-1}}{1+\rho_{\rm{c}}\mathbf{x}^{\rm H} \mathbf{A}^{-1}\mathbf{x}}\\
&=\mathbf{A}^{-1} - \frac{ \rho_{\rm{c}}\mathbf{A}^{-1}\mathbf{x}\mathbf{x}^{\rm H}\mathbf{A}^{-1}  }{1+\rho_{\rm{c}} \sum_{i=1}^N ({|x_i|^2}/{A_{i,i}}) }.
\end{align}
For simplicity, let the scaling factor be defined as $C_0 \coloneqq {\rho_{\rm c}}/\left({1+\rho_{\rm{c}} \sum_{i=1}^N ({|x_i|^2}/{A_{i,i}})}\right)$. Then, the exponent term can be further solved as 
\begin{align}
- {\bf{r}}_{\rm{c}}^{\rm{H}}{\bf{R}}_{{\bf{r}}_{\rm c}|{\bf{x}}}^{ - 1}{{\bf{r}}_{\rm{c}}}&=- {\bf{r}}_{\rm{c}}^{\rm{H}}\left( {\mathbf{A}^{-1} - C_0\mathbf{A}^{-1}\mathbf{x}\mathbf{x}^{\rm H}\mathbf{A}^{-1}}\right){{\bf{r}}_{\rm{c}}}\\
&=- {\bf{r}}_{\rm{c}}^{\rm{H}} {\bf{A}}^{ - 1}{{\bf{r}}_{\rm{c}}} +C_0 {{\bf{r}}_{\rm{c}}^{\rm H}}\mathbf{A}^{-1}\mathbf{x}\mathbf{x}^{\rm H}\mathbf{A}^{-1}{{\bf{r}}_{\rm{c}}} \\
&=-\sum_{i=1}^N\frac{|r_{c,i}|^2}{A_{i,i}} +C_0\bigg\lvert \sum_{i=1}^N \frac{r_{c,i}^{*}x_i}{A_{i,i}}\bigg\rvert^2\\
&=-\sum_{i=1}^N\frac{|r_{c,i}|^2}{N_{0, \rm c} +(1-\rho_{\rm c})|x_i|^2} \\
&\hspace{1cm}+\frac{\rho_{\rm c}\big\lvert \sum_{i=1}^N \frac{r_{c,i}^{*}x_i}{N_{0, \rm c}+(1-\rho_{\rm c})|x_i|^2}\big\rvert^2}{{1+\rho_{\rm{c}} \sum_{i=1}^N \frac{|x_i|^2}{N_{0, \rm c}+(1-\rho_{\rm c})|x_i|^2}}}. \label{eq_exp}
\end{align}
By using the closed forms of the determinant in \eqref{eq_det} and the exponent in \eqref{eq_exp}, the likelihood in \eqref{ML_rule_prob} can be also explicitly solved. Remark that this also holds for $\rho_{\rm c} \in (0,1)$. 

\bibliographystyle{IEEEtran}

\bibliography{References}

\end{document}